\documentclass{article}

\usepackage{microtype}
\usepackage{graphicx}
\usepackage{booktabs} %
\usepackage{listings}
\usepackage{subcaption}

\usepackage{hyperref}
\usepackage[table]{xcolor} %
\newcommand{\ccell}{\cellcolor{gray!30}}
\usepackage{booktabs}

\usepackage[accepted]{icml2025}

\usepackage{amsmath}
\usepackage{amssymb}
\usepackage{mathtools}
\usepackage{amsthm}

\usepackage[capitalize,noabbrev]{cleveref}

\theoremstyle{plain}
\newtheorem{theorem}{Theorem}[section]
\newtheorem{proposition}[theorem]{Proposition}

\theoremstyle{definition}
\newtheorem{definition}[theorem]{Definition}

\theoremstyle{remark}
\newtheorem{remark}[theorem]{Remark}

\usepackage[textsize=tiny]{todonotes}

\newcount\Comments  %
\newcommand{\kibitz}[2]{\ifnum\Comments=1{\color{#1}{#2}}\fi}
\newcommand{\mjc}[1]{\kibitz{blue}{[MJC: #1]}}
\newcommand{\es}[1]{\kibitz{purple}{[ES: #1]}}

\newcount\CommentsAdd  %
\newcommand{\kibitzAdd}[2]{\ifnum\CommentsAdd=1{\color{#1}{#2}}\fi}
\definecolor{english}{rgb}{0.0, 0.5, 0.0}

\icmltitlerunning{LLM-Powered Preference Elicitation in Combinatorial Assignment}

\begin{document}

\twocolumn[
\icmltitle{LLM-Powered Preference Elicitation in Combinatorial Assignment} %

\icmlsetsymbol{equal}{*}

\begin{icmlauthorlist}
\icmlauthor{Ermis Soumalias}{equal,comp}
\icmlauthor{Yanchen Jiang}{equal,yyy}
\icmlauthor{Kehang Zhu}{equal,yyy}
\icmlauthor{Michael Curry}{sch}
\icmlauthor{Sven Seuken}{comp}
\icmlauthor{David C. Parkes}{yyy}
\end{icmlauthorlist}

\icmlaffiliation{yyy}{Harvard University}
\icmlaffiliation{comp}{Department of Informatics, University of Zurich and ETH AI Center}
\icmlaffiliation{sch}{Department of Computer Science, University of Illinois Chicago}

\icmlcorrespondingauthor{Ermis Soumalias}{ermis@if.uzh.ch}
\icmlcorrespondingauthor{Yanchen Jiang}{yanchen\_jiang@g.harvard.edu}
\icmlcorrespondingauthor{Kehang Zhu}{kehangzhu@gmail.com}

\icmlkeywords{Preference Elicitation, Combinatorial Assignment, Large Language Models, LLMs, Course Allocation}

\vskip 0.3in
]
\printAffiliationsAndNotice{\icmlEqualContribution}

\begin{abstract}
We study the potential of large language models (LLMs) as proxies for humans to simplify preference elicitation (PE) in combinatorial assignment. While traditional PE methods rely on iterative queries to capture preferences, LLMs offer a one-shot alternative with reduced human effort. 
We propose a framework for LLM proxies that can work in tandem with SOTA ML-powered preference elicitation schemes. 
Our framework handles the novel challenges introduced by LLMs, such as response variability and increased computational costs. 
We experimentally evaluate the efficiency of LLM proxies against human queries 
in the well-studied course allocation domain, and we investigate the model capabilities required for success. We find that our approach improves allocative efficiency by up to 20\%, and these results are robust across different LLMs and to differences in quality and accuracy of reporting. 
\end{abstract}

\section{Introduction}
\label{intro}
\textit{Preference elicitation (PE)} is essential for effective decision-making in high-dimensional settings. PE methods aim to balance two objectives: 
minimizing the cognitive burden on users by limiting the number of queries they ask,
and maximizing the information obtained about user preferences for subsequent decision-making.
The problem of PE arises across many domains. In this work, we focus on PE in combinatorial assignment (combinatorial auctions, combinatorial course allocation, etc.).

The combinatorial assignment doamins suffers  from the curse of dimensionality (i.e., the bundle space grows exponentially in the number of items).
Moreover, \citet{nisan2006communication} proved that, for arbitrary value functions, achieving full efficiency in combinatorial auctions requires an exponential number of bids. 
To address this, most mechanisms in practice restrict users to report their preferences through structured languages. However, this limits the users' ability to fully articulate their preferences~\cite{nisan2000bidding,sandholm2000improved,fujishima1999taming}.
Thus, the focus has shifted towards \textit{iterative mechanisms}, where bidders interact with the mechanism over a series of rounds, providing only a limited amount of information in each round with the aim of maximizing the efficiency of the final allocation.

In this work, we focus on the \textit{course allocation problem}, a well-studied combinatorial assignment problem \citep{budish2011combinatorial,budish2021can}. In this problem, an educational institution must assign courses to students who often have combinatorial preferences over course bundles, and there is limited seat availability in each course.
\citet{budish2017course} introduced the \textit{Course Match (CM)} mechanism, a significant improvement over previously existing approaches. CM has since been adopted by leading institutions such as Wharton at the ~\citet{Course} and \citet{Coursea}. However, CM's reporting language is both restrictive and cognitively demanding. As a result, students often make reporting errors, which negatively impact the mechanism's performance \citep{budish2021can}.
The iterative \textit{MLCM} mechanism proposed by \citet{soumalias2024machine} addresses these limitations by allowing users to answer adaptively chosen \textit{comparison queries (CQs)}. 
This alleviates reporting errors and leads to significant efficiency gains. 
While CQs are easy for students to answer, a long sequence of iterative queries can still pose a cognitive burden, and students remain inherently limited by the mechanism's reporting language.

We study the use of LLMs as proxies, answering queries for humans guided by a small amount of textual human input, with the goal of both reducing the reporting burden for the users and allowing a richer expression of preferences. 
For instance, consider the following illustrative description of preferences:
\begin{quote}
``I prefer to take courses that are scheduled as closely together as possible so I can have an extra day off. If courses have a laboratory section, I strongly prefer that it be in the morning. Course A and Course B complement each other, and I would prefer to take them together to save time and effort. I do not want to take Course D and E together as they cover similar topics. Nonetheless, I need at least one of them to fulfill my requirements...''
\end{quote}
The textual description encodes combinatorial information about preferences without requiring commitment to any reporting language up-front, and it allows the students to easily express preferences over whole categories of courses (rather than labeling each one). It also provides a more natural elicitation process for students, who may find it easier to express their preferences in free text rather than answering tens of CQs. 

Our aim is to create a framework enabling a mechanism that requires structured input to also leverage such useful yet imprecise natural language input.

A proxy-based approach to preference elicitation has been explored previously, demonstrating the potential of LLMs to simulate human responses~\cite{horton2023large,manning2024automated, park2024generative}.
 We establish that natural language input can be used within SOTA mechanisms for combinatorial assignment.
Our framework takes into account the unique properties of LLMs by using carefully-chosen acquisition functions to generate queries to the LLM proxy, chain of thought to improve the accuracy of the answer, and a noise-robust loss to incorporate that information. In extensive experiments, we verify the significance of all of these decisions, and demonstrate the robustness of our approach to different LLM architectures and differences in the detail and accuracy of preference reporting. Using only a one-shot natural language input per agent, our framework improves allocative efficiency in realistic scenarios by up to 20\%.

\section{Prior work}
In this section, we include prior work with a focus on the central course allocation problem and on LLMs. For a more comprehensive discussion of prior work on PE and machine learning in mechanism design, please see \Cref{app:morepriorwork}.

\paragraph{The course allocation problem}
The course allocation problem is a combinatorial assignment problem, motivated by the real-world challenge of assigning courses in business schools, where students compete for limited seats to build their schedules.
\cite{budish2017course}.
Early methods for performing course assignments at business schools included a draft and a bidding-based mechanism~\cite{sonmez2010course,brams1979prisoners}; however, these created incentives for strategic manipulation and resulted in poor outcomes.
As a combinatorial assignment problem, course allocation is subject to several impossibility results ruling out mechanisms with simultaneous good properties: most notably, combinatorial assignment mechanisms that are ex-post Pareto efficient and strategyproof must be dictatorships~\cite{papai2001strategyproof,hatfield2009strategy}.

To escape this impossibility result, \citet{budish2011combinatorial} proposed a mechanism that satisfies slightly relaxed versions of each of these desiderata: \textit{approximate competitive equilibrium from equal incomes (A-CEEI)}.
At a high level, A-CEEI simulates a competitive equilibrium form equal incomes, but the market need not exactly clear, and students may have slightly different initial endowments of money.

\citet{budish2017course} introduced \textit{Course Match (CM)}, a practical instantiation of A-CEEI tailored to the course allocation problem.
In CM, students report their value for each course, as well as positive or negative pairwise interactions between courses (because the users make these reports via graphical user interface, this is called the ``GUI language'' or ``GUI'').
CM treats these reports as reliably reflecting student preferences, and uses them to heuristically search for an assignment of courses that constitutes an A-CEEI.
CM has been successfully adopted at many leading institutions such as the Wharton School at the University of Pennsylvania and Columbia Business School~\cite{Course,Coursea}.

\paragraph{Machine Learning-powered Course Match}
Prior to CM's adoption by Wharton, \citet{budish2021can} conducted a lab experiment comparing it againt the previously used mechanism, the Bidding Points Auction. Students were happier with their allocation under CM and perceived it as more fair, leading to CM's adoption in practice.
At the same time, \citet{budish2021can} found that students seemed to have trouble with CM's reporting language: 
they make limited use of its features, 
and they sometimes even appear to make outright errors.
Motivated by this problem, \citet{soumalias2024machine} incorporated \textit{machine learning (ML)} into the CM pipeline.
While they start with the same user interface, they also ask students to answer pairwise CQs, which are easier for students. 
These CQs, combined with the reports submitted in the GUI language, are used to train neural network models of student preferences, which guide both query generation and the final allocation process.

\paragraph{Language models as proxies for humans}
A growing body of research explores employing LLMs as proxies for human participants in social and economic studies. 
\citet{horton2023large} provided an early demonstration of this approach by studying how LLM agents, endowed with carefully elicited human preferences, could generate interview responses that closely mirrored those given by human subjects. Building on this work, \citet{park2024generative} investigated whether LLM-based proxies could replicate qualitative interview responses as accurately as real participants. They found that LLM proxies matched human responses with an 85\% accuracy, suggesting substantial potential for LLMs to substitute for
human participants in certain contexts.

Subsequent work expanded the scope of LLM agents to broader social and economic scenarios. For example, \citet{brand2023using} and \citet{manning2024automated} studied how these models perform in auctions, negotiations, and marketing environments, providing further evidence that LLMs can effectively mimic human decision-making. More recently, researchers have applied LLMs as autonomous pricing agents for companies \cite{fish2024algorithmic} and as synthetic participants in auction design \cite{zhuevidence}.

In concurrent work, \citet{Huang25Accelerated} explore the use of LLMs as proxies in combinatorial \textit{allocation} domains, in contrast to our focus on large-scale combinatorial assignment. This distinction introduces a different set of challenges, resulting in a different framework for PE.

\paragraph{Mechanism Design for LLMs.}
A related area pioneered by \citet{duetting2024mdforllms} is that of mechanism design for LLMs.
In this setting, agents---typically advertisers---compete to influence an LLM's reply to a user's query, aiming to better represent their interests.
\citet{soumalias2024truthful} introduce a truthful mechanism based on importance sampling that converges to the optimal distribution. 
\citet{mohammad2024rag} leverage retrieval-augmented generation to create an auction where a pre-generated ads are probabilistically retrieved for each discourse segment, according to both their bid and relevance. 
\citet{bergemann2024datadrivenmd} explore an extension of this problem in which agents possess both private types and signals.

\section{Preference Elicitation Framework}
\label{sec:PE_framework}

\begin{figure*}[h]
    \centering
\includegraphics[width=0.82\linewidth]{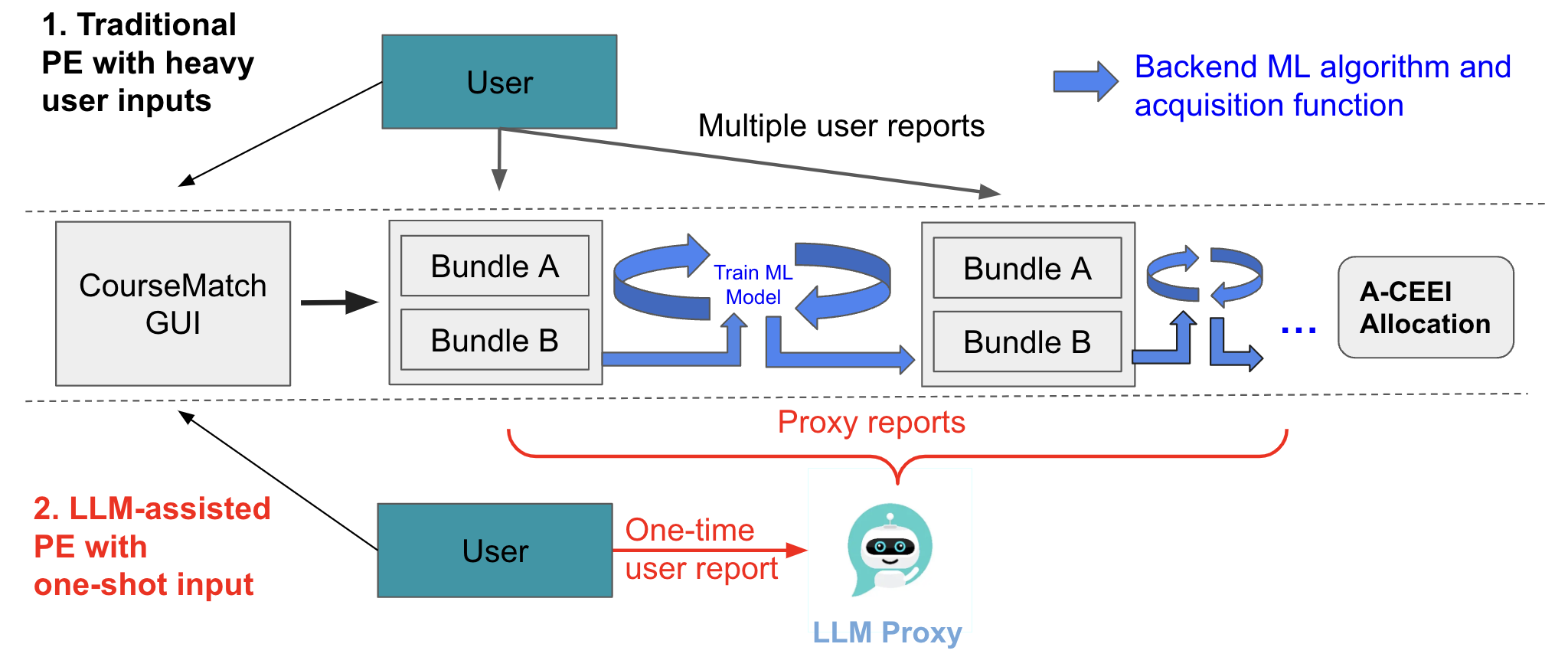}
    \caption{A schematic flowchart comparing \citet{Soumalias2024MLCCA} with the LLM proxy algorithm in this paper.} 
    \label{fig:flowchart}
\end{figure*}

We build on the MLCM mechanism \citep{soumalias2024machine}, which initializes a \textit{cardinal} ML model of each student's value function based on her GUI reports.
Following the Bradley-Terry model~\cite{bradleyterry1952rank}, this cardinal model is further trained using the student's responses to pairwise \textit{comparison queries (CQs).}
In MLCM, the CQs are determined by a specialized acquisition function that maintains an ordinal list of schedules based on the student's previous answers. 
Finally, an A-CEEI is calculated, treating each student's ML model as her value function.

We extend MLCM by designing a framework that allows students to provide one-shot natural language input, enabling an LLM proxy to answer CQs on their behalf. This approach lets students express their preferences naturally and eliminates the need for iterative interaction with the mechanism. Additionally, the LLM proxy can handle a large number of CQs without imposing any cognitive burden on the student, unlocking potential efficiency gains.
While our implementation focuses on course allocation, in \Cref{sec:beyond_course_allocation} we discuss its applicability to other combinatorial allocation settings.

\subsection{Noise Robust Loss}
\label{sec:NoiseRobustLoss}
LLM responses provide a valuable but inherently noisy signal for PE. 
To combat this this issue,  we leverage the noise-robust \textit{generalized cross-entropy} (GCE) loss \citep{Zhang18Generalized} instead of the conventional binary cross entropy loss when training student models on pairwise CQs. This allows us to handle the LLM noise effectively, as formalized in the following result:

\begin{proposition} \label{prop:true_minimizer}
In our framework, as long as the LLM proxy accuracy is over $50$\%, 
the student's true valuation function is a minimizer of the training loss in the noisy dataset produced by the LLM proxy. 
\end{proposition}

\begin{proof}
The proof is deferred to  \Cref{sec:PE_framework_details}.
\end{proof}

\begin{remark}
In our experiments in \Cref{sec:experimets}, for \textit{all} LLM architectures tested, the LLM proxy accuracy is over $71$\% (see \Cref{tab:llm_architecture_ablation} as well as \Cref{app:cqaccuracy}). 
The proposition is thus relevant to our system's real-world performance.
In \Cref{sec:gce_vs_bce} we show the practical significance of GCE, as it doubles the effectiveness of our framework over BCE.
\end{remark}

\subsection{Comparison Query Generation}
\label{sec:CQGeneration}

Although LLM proxies can respond to any number of queries without a cognitive cost to the student, each response still imposes a computational cost to the mechanism. 
Thus, the choice of the \textit{acquisition function} that selects the CQs is vital to both reduce computational costs and improve allocative efficiency.

Although we do not fine-tune language models, our preference learning problem from comparison queries has many similarities to reinforcement learning from human feedback (RLHF)~\cite{Bai22Training}; 
(1) We use CQs, and there is an inherent connection between our loss function and the Bradley-Terry model (see \Cref{sec:training_alg} for details) 
(2) For each element of each CQ, we can choose among an enormous set of possible elements. 
(3) We face a similar \textit{Bayesian Optimization task}; 
our goal is not to improve learning performance in itself. 
Rather, our goal is to maximize the student's true value for the bundle she receives under her final trained model. 
This is similar to RLHF, where the goal is to maximize the user's expected satisfaction with the generated answers \citep{Bai22Training}.

Given this connection to RLHF, we draw on recent work exploring different acquisition functions for this task \citep{Dwaracherla24Efficient}.
Because the best-performing acquisition functions in that line of research also leverage epistemic uncertainty, 
we create \emph{epistemic Monotone Value Neural Networks (eMVNNs)},
an extension of MVNNs \citep{weissteiner2022monotone} that can also represent the epistemic uncertainty of the network. 
MVNNs are a specialized architecture for combinatorial allocation that incorporates at a structural level the prior information that value functions are monotone and that the value of the empty bundle is zero. 
The best performance is achieved when queries are generated using Double Thompson Sampling, with uncertainty captured by the eMVNN.
For more details on MVNNs and our eMVNNS, please see \cref{sec:MVNNs,sec:eMVNNs}.

\begin{remark}
\Cref{sec:acquisition_function_comparison} demonstrates the practical significance of modeling epistemic uncertainty and choosing the appropriate acquisition function,
as they double performance versus \emph{all} tested alternatives.
\end{remark}

\section{Large Language Models} \label{sec:llms}

In this section, we describe our use of LLMs: what information they are given, and how we generate outputs.
We emphasize that we use LLMs for \emph{\mbox{two distinct purposes}}.
The first and main purpose is as a \emph{\mbox{proxy}} to answer CQs based on free-text preference descriptions provided by the students; these LLM proxies are part of our framework and would be used if it were actually deployed (\Cref{sec:llmproxy}).
The second is as a \emph{{\mbox{surrogate student}}} to generate free-text preference descriptions from some underlying utility model; these are part of our experimental methodology, but would not be used if the mechanism were actually deployed---instead, real students would produce free-text preference descriptions based on their actual preferences (\Cref{sec:simulatedstudent}).

\subsection{LLMs as Proxies for Comparison Queries} \label{sec:llmproxy}

Given students' textual input, LLMs can extract underlying preferences and answer CQs in 
place of the student, reducing cognitive load.
However, in our experiments, we observed that directly prompting LLMs to answer CQs often leads to inaccuracies, particularly due to hallucinations when the input text is long and contains many courses. 

To address this issue, we implemented a structured \textit{chain-of-thought (CoT)} reasoning approach \citep{wei2022chain}. The reasoning process guides the LLM step by step: first recalling the items in the bundle, then systematically addressing key components such as preferences, complements, and substitutes, before arriving at a final choice. This approach reduces hallucinations and improves the accuracy of the LLM's responses.
Note that real students do not have to perform any CoT, as this is done by their LLM proxy.

We designed a structured template using XML tags to organize the reasoning process into sections. A sketch of the template follows, where ellipses (`...') indicated omitted details for brevity:

\vspace{-0.2cm}
\begin{quote}
\small{
\texttt{$<$PREFERENCES$>$}

Bundle A: [Recall the courses in the bundle and list matching preferences. e.g. ...] 

Bundle B: [...] 

\texttt{$<$/PREFERENCES$>$}

\texttt{$<$COMPLEMENTS$>$} [...] \texttt{$<$/COMPLEMENTS$>$}

\texttt{$<$SUBSTITUTES$>$} [...] \texttt{$<$/SUBSTITUTES$>$}

\texttt{$<$REASONING$>$} [...] \texttt{$<$/REASONING$>$}

\texttt{$<$CHOICE$>$} Bundle X \texttt{$<$/CHOICE$>$}
}
\end{quote}
\vspace{-0.3cm}

The full template, along with examples and detailed explanations, is provided in Appendix \ref{appendix:cqprompt}. An experimental analysis showing the importance of CoT is in \Cref{sec:cotexperiments}.

\subsection{Simulated Student Responses}
\label{sec:simulatedstudent}
We also use LLMs to simulate how students might describe their preferences in free text before answering CQs.
In our experiments, Llama-3.1 8b \citep{dubey2024llama}\footnote{
In \Cref{sec:llm_architecture_ablation} we test our framework with various LLM architectures and observe similarly strong results. 
} is given numerical utilities (including complements and substitutes) from \citet{soumalias2024machine} and instructed to write a first-person narrative that omits explicit numbers.
The LLM was instructed to adopt the role of a student describing their preferences qualitatively, avoiding direct numerical outputs.
A sketch of the template follows, where ellipses (`...') indicated omitted details for brevity:

\vspace{-0.3cm}
\noindent %
\begin{quote}
\textit{
Please act as a student describing their course preferences for the upcoming semester.}

\textit{Write a detailed, first-person paragraph about your preferences based on the following information: [...]}

\textit{Aim for qualitatively detailed description and avoid saying exact numerical values [...]}
\end{quote}
\vspace{-0.3cm}
The full prompt template is in Appendix \ref{appendix:prompt}, with a simulated student response example in Appendix \ref{appendix:examplestudent}.  Further experiments (Appendix \ref{sec:appendix-conciseness}) show the framework remains effective with when responses are made more concise.

\section{Experiments} \label{sec:experimets}
In this section, we experimentally evaluate our LLM-powered PE framework.

\subsection{Experiment Setup} \label{sec:experiment_setup}

We use the Course Allocation Simulator from \citet{soumalias2024machine} because it offers a unique combination of extensive preference data and validated error rates from real lab experiments.
Specifically, its configuration closely mirrors the results in \citet{budish2021can}, matching both the frequency and severity of student mistakes within a 1\% margin. 
To focus solely on learning performance, we set all course capacities to infinity, ensuring that each student can receive any combination of five courses.
In line with our LLM-based design (\Cref{sec:llmproxy,sec:simulatedstudent}), we first convert each synthetic student’s cardinal preferences into a free-text description using Llama-3.1.
These narratives are then passed to the LLM proxy, which answers comparison queries by interpreting the student’s stated preferences.
Further implementation details, including our hyperparameter optimization (HPO) protocol, are provided in \Cref{sec:HPO}.

\subsection{Efficiency Results} \label{sec:experiment_results_efficiency}
In \Cref{fig:allocated_bundle_value} we plot the student's allocated bundle value against the number of CQs answered by her proxy LLM.\footnote{Note that in our framework, each student provides only a single piece of natural language input in a one-shot manner. Thus, the x-axis reflects the computational cost for the mechanism rather than the cognitive effort required from the student.}
We normalize the value of the bundle each student receives by the bundle that the student would have received based on her initial GUI reports.
Our results demonstrate that, with just one natural language input from the student, 
our framework increases allocative value by $19.3$\%.
Furthermore, the student prefers her allocation under our framework in $74$\% percent of the cases. 
For context, in this same simulation, MLCM, the SOTA mechanism for course allocation, 
increases allocative efficiency by $14.2$\% using $20$ student-answered CQs. 
This highlights the effectiveness of our framework in improving allocation outcomes with significantly lower student effort.

\begin{figure*}[h]
    \centering
            \begin{subfigure}[t]{0.49\textwidth} %
        \centering
        \includegraphics[width=0.9\linewidth]{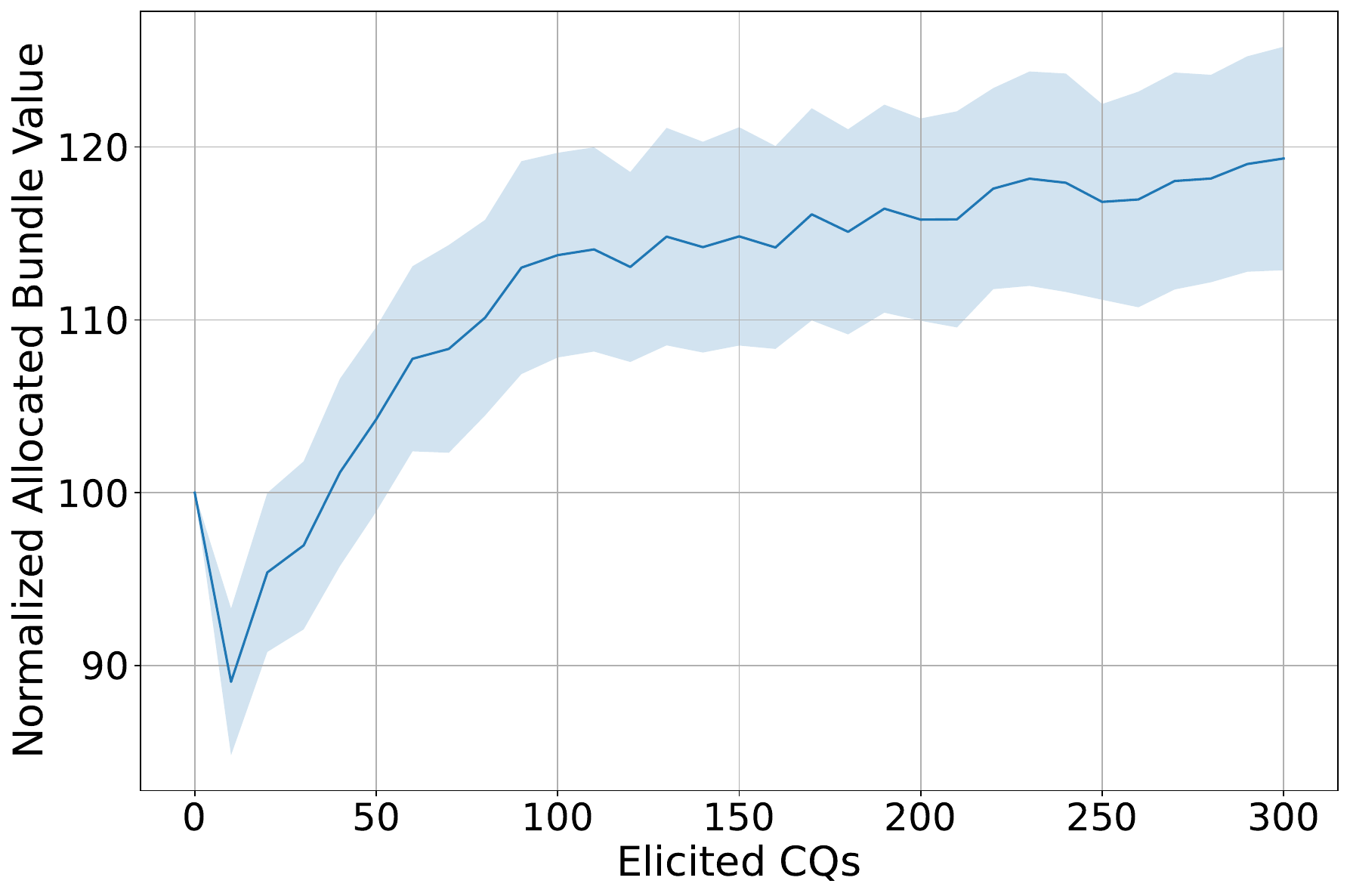}
        \caption{Normalized allocated bundle}
        \label{fig:allocated_bundle_value}
    \end{subfigure}
    \hfill
    \begin{subfigure}[t]{0.49\textwidth} %
        \centering
        \includegraphics[width=0.9\linewidth]{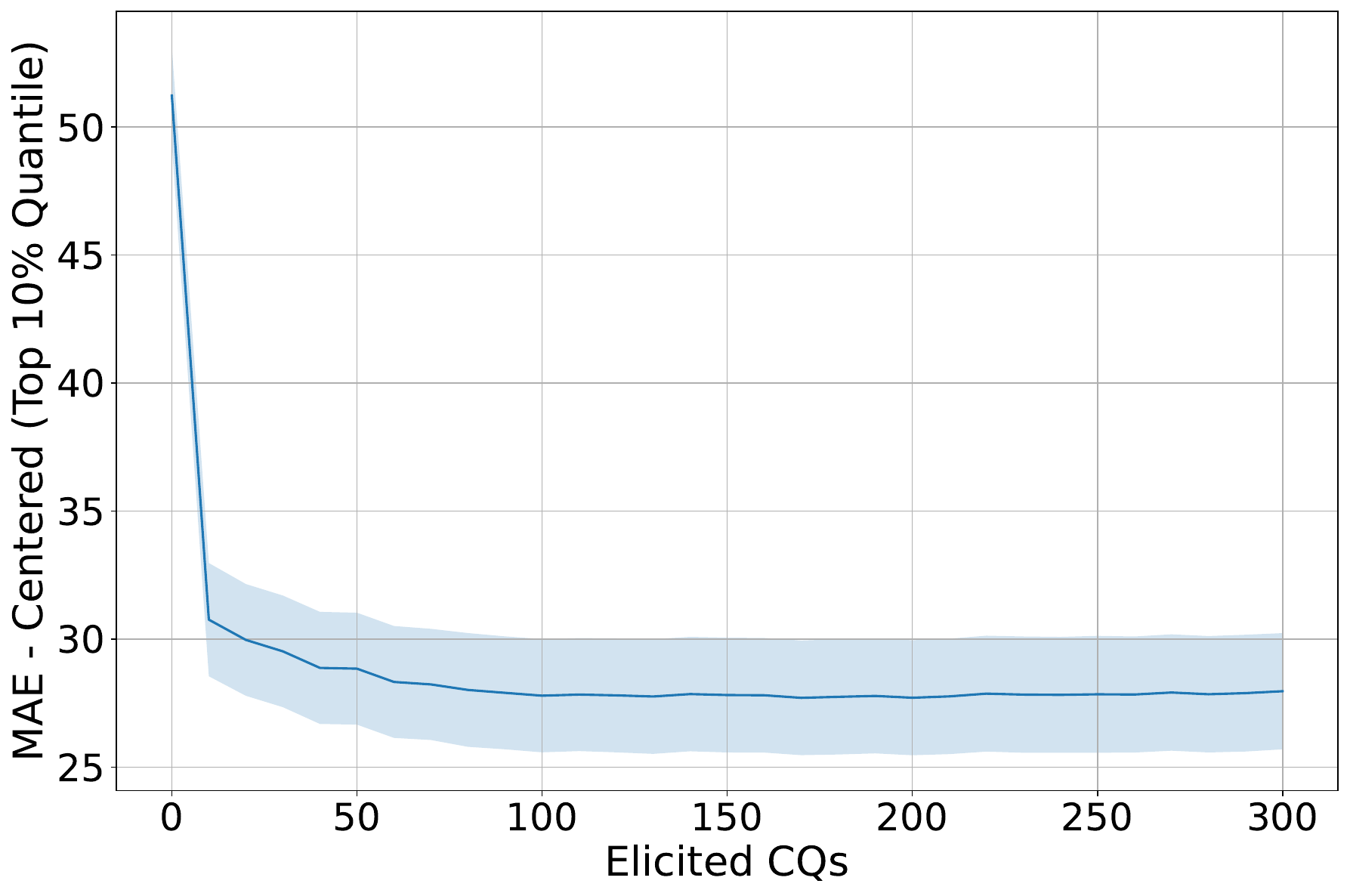}
        \caption{Centered MAE.}
        \label{fig:mae_centered}
    \end{subfigure}
    \caption{Comparison of (a) normalized allocated bundle value and (b) centered MAE, both as functions of the number of LLM-answered CQs. Shown are averages over 100 instances including 95\% CIs.}
\end{figure*}

\subsection{Learning Performance} \label{sec:experiment_results_learning}

\begin{table}[ht]
\centering
\label{tab:learning_metrics}
\begin{tabular}{l r c c c }
\toprule
\textbf{Metric} & \multicolumn{2}{c}{\textbf{Whole Test Set}} & \multicolumn{2}{c}{\textbf{Top 10\% Quantile}}  \\
\cmidrule(lr){2-3} \cmidrule(lr){4-5} 
& \textbf{Before} & \textbf{After} & \textbf{Before} & \textbf{After}  \\
\midrule
$\text{MAE}_C$      & $46.14$  & \ccell $30.69$ & $51.23$ & \ccell $27.96$   \\
$\text{KT}$         & $0.266$  & \ccell $0.389$ & $0.121$ & \ccell $0.283$  \\
$\text{MSE}_C$      & $3475$   & \ccell $1942$  & $4172$  & \ccell $1500$   \\
$R^2_C$             & $-0.62$  & \ccell $0.22$  & $-2.50$ & \ccell $0.14$  \\
\bottomrule
\end{tabular}
\caption{Comparison of learning metrics before and after applying our framework. 
Each metric is reported for the whole test set and the top 10\% quantile in terms of student value. 
Shown are averages over $100$ runs.
Winners based on a paired t-test with $\alpha=1\%$ marked in gray.}
\label{tab:learning_comparison}
\end{table}

To assess our framework's impact on the generalization performance of a student's ML model, in \Cref{fig:mae_centered} we plot ($\text{MAE}_C$), a shift invariant measure of the mean absolute error  
against the number of LLM-answered CQs.

We focus on a shift-invariant regression metric because the allocation algorithm assigns to each student the bundle with the highest predicted value--based on her ML model--that is attainable for her. 
Thus, learning a student's value function up to a constant shift suffices, since it results in the same (optimal) allocation as the true value function under any set of constraints.
Similarly, we focus on the student's top quantile of bundles, as these are the most critical for ensuring high-value allocations. Our results demonstrate that our framework leads to an immediate and significant reduction in $\text{MAE}_C$, nearly halving its original value.

In \Cref{tab:learning_comparison}, we further analyze the impact of our framework on the generalization performance of a student's ML model using both rank-based and shift-invariant learning metrics. 
We report results for the top quantile to assess performance in the most consequential part of the bundle space, as well as across the entire bundle space to measure overall generalization.
In \Cref{sec:learning_comparison_details} we present an expanded version of this table and find similarly strong results for more quantiles.

\Cref{tab:learning_comparison} demonstrates that our framework—using just a \textit{single} natural language input from the student—substantially enhances all learning metrics, both rank-based and shift-invariant. This improvement is even more pronounced in the top quantile, the most critical region from a Bayesian optimization perspective. Moreover, the performance gains achieved by our framework are statistically significant at the $\alpha = 1$\% confidence level for \emph{all} metrics and quantiles.

Taken together, our results in \Cref{sec:experiment_results_efficiency,sec:experiment_results_learning} show that our LLM-powered PE framework leads to vastly improved learning performance and allocative efficiency.

\subsection{Robustness to Students' Mistakes}

In this section, we evaluate our framework's performance when changing the severity of the students' mistakes in the GUI language. 
To do this, we multiply all parameters affecting their mistake profile in the simulator by a constant  $\gamma$. 
For $\gamma >1$, students make more mistakes than in the default profile, and for $\gamma < 1$ the opposite is true.  
Importantly, we do not change any of the hyperparameters of our framework compared to their optimized values for our default setting.

In \Cref{tab:noise_ablation}, we present the normalized allocation value, and percentage of instances where our framework increases allocative value (over just the GUI reports) as a function of the mistake multiplier. 
We observe that our framework improves average allocation even for $\gamma = 0.5$.
Importantly, mistakes do not scale linearly; in that case, students make about half the amount of mistakes reported in \citet{budish2021can}, and the severity of those mistakes is reduced by over $80$\% \citep[Section 7.4]{soumalias2024machine}.
Overall, these results demonstrate that our framework is robust to varying mistake severity, consistently improving allocation outcomes even under significantly reduced or amplified error profiles.

\subsection{Evaluating the Effectiveness of Chain of Thought}
\label{sec:cotexperiments}
In this section, we evaluate the effectiveness of CoT prompting.  
We compare using our framework with and without the CoT reasoning teamplate introduced in \Cref{sec:llmproxy}.

Table~\ref{tab:allocated_bundle_values} reveals that incorporating CoT significantly improves performance across all metrics: the accuracy of the LLM proxy increases from 59\% to 72\%, and the normalized allocated bundle values show statistically significant gains (visualized in Figure~\ref{fig:llm_cq_value_cot}). 
Without CoT, despite using optimal hyperparameters, performance improvements remain statistically insignificant. This comprehensive improvement in accuracy is further illustrated in Figure~\ref{fig:llm_cq_correctness_cot} (Appendix~\ref{sec:additional_results}), demonstrating the critical role of CoT reasoning in enhancing performance.

\begin{figure}[h]
    \centering
    \includegraphics[width=\linewidth]{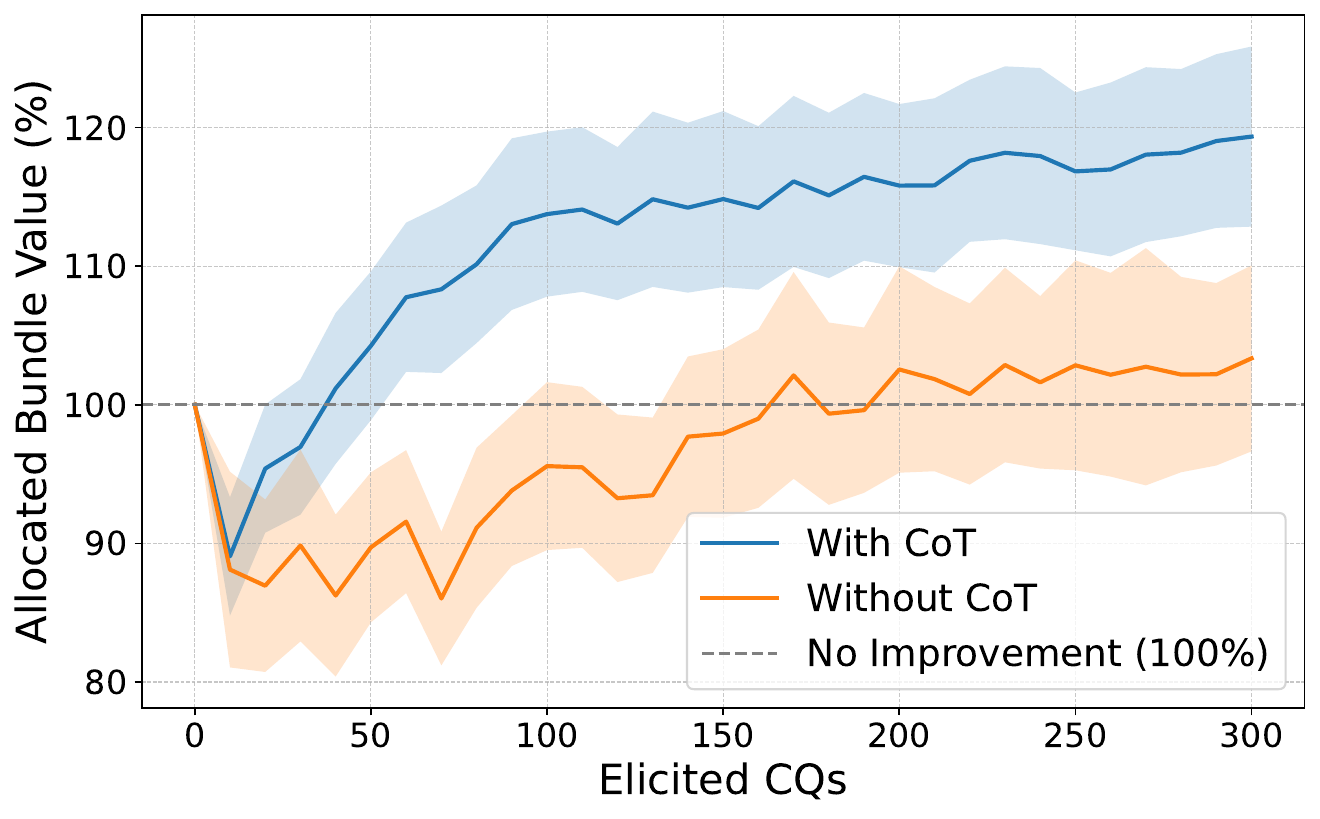}
    \vspace{-0.5cm}
    \caption{Comparison of the normalized allocated bundle value with and without the Chain-of-Thought (CoT) reasoning approach using LLaMA 3.1 8b model. The use of CoT results in a statistically significant improvement in the allocated bundle value.}
    \label{fig:llm_cq_value_cot}
\end{figure}

\begin{table}[ht]
\centering
\small
\setlength{\tabcolsep}{4pt}
\begin{tabular}{l c c c c c}
\toprule
\textbf{Version} & \textbf{Normalized} & \multicolumn{2}{c}{\textbf{\% of Runs}} & \textbf{P-value} & \textbf{Accuracy} \\
\cline{3-4}
& \textbf{Value} \(\pm\ CI\)
& \textbf{Better}
& \textbf{Worse} 
& 
& \textbf{ (\%) }\\
\midrule
w/ CoT & 119.34 ± 6.50 & 74 & 26 & 6.8e-8 & 72.3 \\
w/o CoT & 103.34 ± 6.71 & 60 & 40 & 3.3e-1 & 59.6 \\
\bottomrule
\end{tabular}
\caption{Allocated bundle values, statistical significance, and LLM accuracy with and without CoT.
}
\label{tab:allocated_bundle_values}.
\end{table}

\begin{table}[ht]
\centering
\begin{tabular}{l c c c c}
\toprule
\textbf{Noise} & \textbf{Normalized} & \multicolumn{2}{c}{\textbf{\% of Runs}} & \textbf{P-value} \\
\cline{3-4}
\textbf{Mult.}
& \textbf{Value} \(\pm\ CI\)
& \textbf{Better}
& \textbf{Worse}
& \\
\midrule
$0.5$  & \(103.93 \pm 6.76\)  & 52 & 48 & $0.132$ \\
$0.75$ & \(108.66 \pm 7.31\)  & 66 & 34 & $0.013$ \\
$0.9$  & \(106.08 \pm 8.22 \) & 54  & 46   &  $0.079$ \\
\midrule
$1$    & \(119.34 \pm 6.46\) & 74 & 26 & 3.41e-8 \\
\midrule
$1.1$    & \( 113.05 \pm 8.39 \) & 66 & 34 & $0.002$ \\
$1.25$    & \(122.64 \pm 8.12\) & 74 & 26 & 9.37e-7 \\
\bottomrule
\end{tabular}
\caption{Normalized allocated bundle value with our framework as a function of the noise multiplier used. 
Shows are averages over $100$ runs over 100 runs for the default noise level and 50 runs for all others, including $95$\% CIs.
We also show the \% of cases where our framework improves a student's allocation, and the significance of these results.}
\label{tab:noise_ablation}
\end{table}

\subsection{LLM architecture ablation test} \label{sec:llm_architecture_ablation}
In this section, we evaluate our framework's robustness to LLM architectural variations. As detailed in \Cref{sec:llms}, our experiments employ two distinct LLMs: 
the LLM proxy used by our framework to finetune the student's ML model (\Cref{sec:llmproxy}), and the LLM simulating student textual input to our framework (\Cref{sec:simulatedstudent}).

\Cref{tab:llm_architecture_ablation} presents performance results across different LLM architectures for both student simulation and CQ answering. The results demonstrate remarkable consistency: allocation value improves by 19--22\% (with high statistical significance) across all configurations, with student allocations improving in 80--82\% of cases. This robust performance stems from consistently high CQ accuracy (72--75\%) by the LLM proxy across all configurations. As expected, commercial, larger models outperform open-sourced smaller models, though we primarily use the latter for cost efficiency. These findings demonstrate our framework's adaptability to both architectural variations and diverse natural language inputs, with further analysis of CQ accuracy sensitivity presented in \Cref{sec:cq_accuracy_robustness_test}.

\begin{table*}[h]
\centering
\begin{tabular}{l l c c c c}
\toprule
\textbf{LLM} & \textbf{LLM} & \textbf{Allocation Value} & \textbf{\% Better} & \textbf{CQ Accuracy} & \textbf{P-value} \\
\textbf{Surrogate} & \textbf{Proxy} & \(\pm\ CI\) & & \textbf{(\%)} & \\
\cmidrule(lr){1-2} \cmidrule(lr){3-3} \cmidrule(lr){4-4} \cmidrule(lr){5-5} \cmidrule(lr){6-6}
ChatGPT & ChatGPT & \(120.67 \pm 8.44 \) & 82 & $72.11$ & 9.84e-6  \\
LLaMA   & ChatGPT & \(121.72 \pm 8.76 \) & 82 & $73.70$  & 8.38e-6 \\
LLaMA   & LLaMA   & \(119.34 \pm 6.50 \) & 74  & $71.69$  & 6.8e-8 \\
\bottomrule
\end{tabular}
\caption{LLM-powered PE framework performance under different LLM architectures. 
Each setup involves one LLM acting as a surrogate student, simulating the student's natural language input to the mechanism and another LLM proxy answering CQs based on the input. We report the normalized allocated bundle value, the percentage of instances where our framework improves allocative efficiency, the CQ accuracy of the CQ answerer LLM, and the statistical significance of the allocated bundle value results (p-value). 
Shown are averages over at least $50$ runs for each configuration, including 95\% confidence intervals (CIs).
}
\label{tab:llm_architecture_ablation}
\end{table*}

\subsection{Comparing BCE against GCE} \label{sec:gce_vs_bce}
\begin{figure}[h]
\centering
\includegraphics[width=\linewidth]{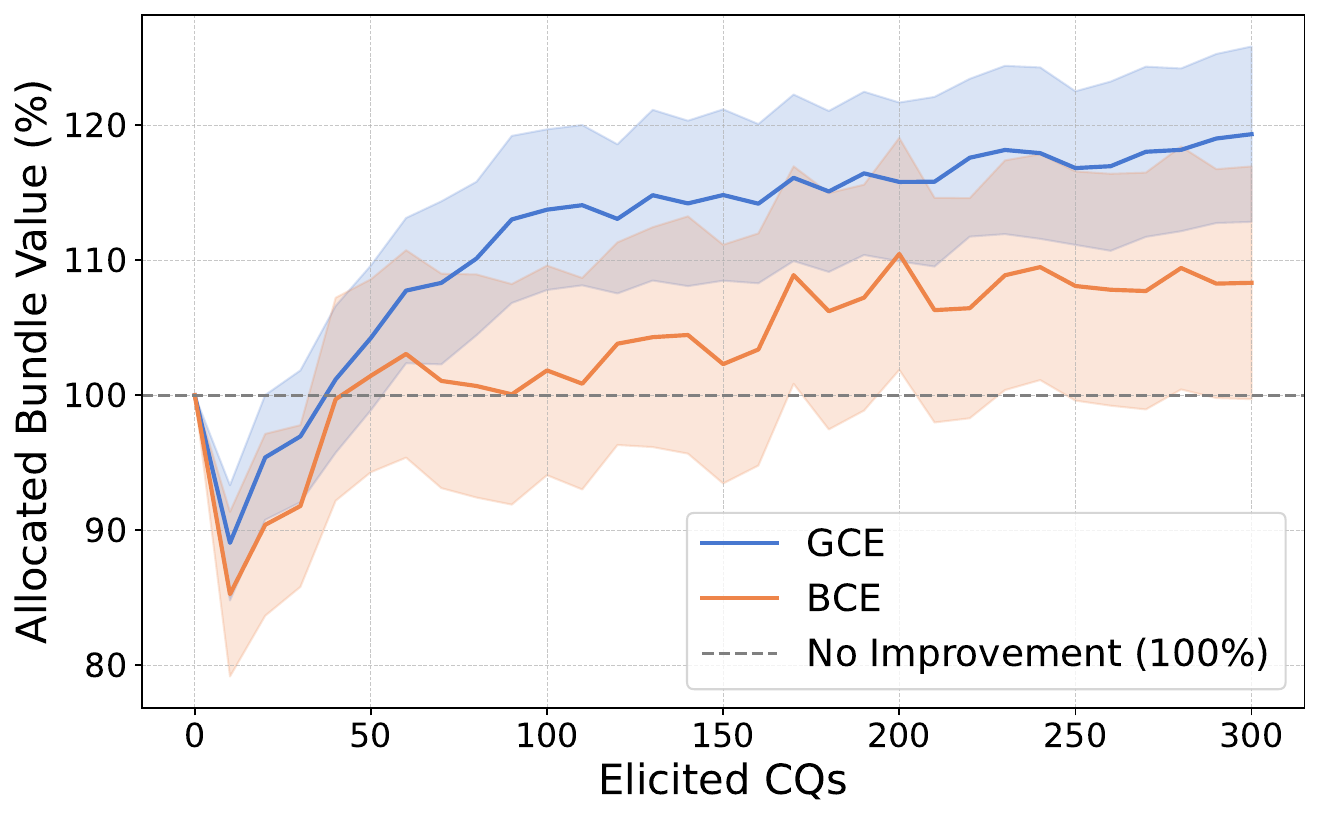}
\vspace{-0.5cm}
\caption{Normalized  allocated bundle value as a function of the number of LLM-answered CQs. 
We compare GCE versus BCE as the training loss for the LLM-answered CQs. 
Shown are averages over 50 instances including 95\% CIs. \es{This figure can also be moved to the appendix and only show the table.}
}
\label{fig:gce_vs_bce}
\end{figure}

In this section, we study the effect of the GCE loss in improving our framework's performance.
We compare allocation value under the default model parametrization (\Cref{sec:HPO}) with a version using standard binary cross-entropy (BCE) loss. Both versions were allocated the same compute time for hyperparameter optimization (HPO), as detailed in \Cref{sec:HPO}.

Similar to our main results in \Cref{sec:experiment_results_efficiency},  in \Cref{fig:allocated_bundle_value} we plot the student's allocated bundle value, as a function of the number of CQs answered by the student's LLM proxy. 
We observe that for any number of CQs, GCE outperforms BCE. 
Importantly, GCE increases allocation value by an average of $19.3$\% using 300 CQs, 
while BCE only achieves an increase of $8.3$\%, and GCE has a much more significant p-value ($6.84 \times 10^{-8}$).

\subsection{Comparing Different Acquisition Functions} \label{sec:acquisition_function_comparison}

In this section, we evaluate the importance of the acquisition function in our framework's performance.
We compare the allocation value across different acquisition functions, using the default parametrization of our framework as described in \Cref{sec:HPO}, while varying only the acquisition function responsible for selecting the  CQs.

\Cref{tab:acquisition_functions} highlights the effectiveness of Double Thompson Sampling, which doubles the improvement in allocation value compared to the second-best acquisition function—from $10.95$\% to $19.34$\%. 
Notably, Double Thompson Sampling is also recognized as a superior acquisition function in RLHF settings \citep{Dwaracherla24Efficient}.
Interestingly, all other acquisition functions tested, including random selection, exhibit nearly identical performance.
\mjc{i commented out fig 5 and the text referencing it.}

\begin{table}[ht]
\centering
\begin{tabular}{l c c c c}
\toprule
\textbf{Acqu.} & \textbf{Normalized} & \multicolumn{2}{c}{\textbf{\% Runs}} & \textbf{P-value} \\
\cline{3-4}
\textbf{Function}
& \textbf{Value} \(\pm\ CI\)
& \textbf{Better}
& \textbf{Worse} \\
\midrule
DoubleTS & 119.34 ± 6.50 & 74 & 26 & 6.8e-8 \\
Infomax & 110.51 ± 8.05 & 66 & 34 & 0.014 \\
Boltzmann & 110.95 ± 8.86 & 60 & 40 & 0.019 \\
Random & 109.86 ± 9.00 & 64 & 36 & 0.037 \\

\bottomrule
\end{tabular}
\caption{Normalized allocated bundle value with our framework for various acquisition functions. 
Shows are averages including $95$\% CIs.
We also show the \% of cases where our framework improves a student's allocation, and the significance of these results.
}
\label{tab:acquisition_functions}
\end{table}

\subsection{LLM Accuracy Robustness Test} \label{sec:cq_accuracy_robustness_test}
In this section, we evaluate our framework's robustness to variations in the accuracy of the proxy LLM. Given the consistently high performance (72--75\%) of tested LLM architectures (\Cref{sec:llm_architecture_ablation}), we employ \emph{simulated LLMs} to explore a broader accuracy range. 
Specifically, for each CQ chosen by the acquisition function, we provide the correct reply with probability equal to the simulated LLM accuracy, assuming i.i.d. mistakes.\footnote{
\Cref{fig:llm_cq_accuracy} motivates this choice, as the LLM accuracy remains roughly constant throughout the process. 
}
Importantly, for this test, we do not perform HPO, but instead use the values for all hyperparameters determined in our default setting.

\Cref{tab:cq_accuracy_ablation} demonstrates our framework's robustness: it continues to improve allocation value even when LLM accuracy drops to 60\%.
This robustness is particularly noteworthy given that all tested LLMs maintain 72--75\% accuracy (\Cref{sec:llm_architecture_ablation}). For comparison, the state-of-the-art mechanism by \citet{soumalias2024machine} loses half its efficacy with just 16\% student error rate in CQ answers.\footnote{In their mechanism, the CQs are answered by real students, and prior work suggests that near-perfect accuracy by real students can generally be assumed.}

\begin{table}[ht]
\centering
\begin{tabular}{l c c c c}
\toprule
\textbf{LLM CQ} & \textbf{Normalized} & \multicolumn{2}{c}{\textbf{\% of Runs}} & \textbf{P-value} \\
\cline{3-4}
\textbf{Accuracy}
& \textbf{Value} \(\pm\ CI\)
& \textbf{Better}
& \textbf{Worse}
& \\
\midrule
$70$\% & \(126.50 \pm 1.79\)  & $84.5$ & $14.6$ &  2.6e-135 \\
$65$\% & \(118.03 \pm 1.85\)  & $71.9$ & $27.6$  & 7.5e-70 \\
$60$\% & \(104.29 \pm 1.75 \) & $52.9$ & $47.1$  & 8.93e-7 \\
$55$\% & \(85.76 \pm 1.65 \)  & $26.1$ & $73.9$ &  $1.0$ \\
\bottomrule
\end{tabular}
\caption{Normalized allocated bundle value with our framework for various levels of \emph{simulated} LLM accuracy. 
Shows are averages over $1000$ runs including $95$\% CIs.
We also show the \% of cases where our framework improves a student's allocation, and the significance of these results.}
\label{tab:cq_accuracy_ablation}
\end{table}

\section{Discussion and Conclusion}

\subsection{Limitations of the Language Model}
Despite the clear benefits of using \textit{large language models (LLMs)} to elicit preferences and answer \textit{comparison queries (CQs)}, we will discuss two important limitations here. First, while the model frequently demonstrates solid general reasoning skills, we have observed gaps in its quantitative reasoning: when comparisons hinge on numerical trade-offs or multi-attribute scoring, the model’s accuracy can suffer. Improving the model’s ability to precisely handle numerical relationships is a key challenge for future research.

Second, although LLMs can effectively understand user preferences in a broad, qualitative sense, their grasp of parsing user requests and assigning ordinal value among diverse items may remain imperfect. 
Nonetheless, we anticipate that continued advancements in language model architectures and fine-tuning techniques will help bridge this gap.  

\subsection{Cost Considerations}
A further practical concern with current LLMs is the \emph{cost} associated with their use. In our experiments, each student required roughly 300 comparison queries, amounting to an average of 0.1 million input tokens and nearly the same amount of output tokens in the CoT response. At scale, these token counts could become prohibitively expensive, especially for advanced models with higher per-token pricing. Luckily, our experiments show that we are not reliant on advanced models, and even using open-sourced models with smaller parameter sizes like LLaMA 3.1 8b, we are still able to achieve significant results.

Encouragingly, both the capabilities and cost structures of LLMs are \emph{rapidly evolving}. For instance, in a pilot study using the GPT-o1 model \cite{zhong2024evaluation}, we observed 
a substantial increase in CQ accuracy to over $80$\%.
However, this experiment was halted due to the model’s high price. 
These developments suggest that, as the LLM ecosystem becomes more competitive, overall costs will likely decrease without sacrificing reasoning quality.

\subsection{Applications Beyond Course Allocation} \label{sec:beyond_course_allocation}

Although our experiments focus on the course allocation domain, the same LLM-based one-shot preference elicitation approach naturally extends to a wide array of settings in which agents have complex preferences over combinatorial outcomes. 
Our framework is broadly applicable to other iterative combinatorial allocation mechanisms that use trained ML models to guide query generation, such as MLCA, ML-CCA, and BOCA \citep{brero2021workingpaper,soumalias2024machine,weissteiner2023bayesian}; 
as detailed in \Cref{sec:experiment_results_learning}, our framework improves the generalization of agents' ML models, particularly in critical regions shown to enhance allocative outcomes \citep{soumalias2024pricesbidsvalueseverything}.

Incorporating LLM proxies into mechanisms offers a unique opportunity to streamline the elicitation process by reducing the time and cognitive effort required from participants, while also leading to more effective and equitable allocations. As LLMs continue to improve in both capability and affordability, they present a promising new approach for making mechanisms more practical and accessible in real-world applications.

\section*{Broader Impact}
\mjc{broader impact doesn't count against page limit}
Our framework improves allocative value, while also reducing cognitive load on the participating agents. 
This has the potential of making complicated mechanisms more efficient, fair and accessible. 
The course assignment mechanism of \citet{budish2017course} was carefully tested before adoption to ensure that it actually improved student welfare and to understand its limitations. The same would be essential for any future improvements to the mechanism, such as our LLM proxy technique. We therefore see the potential for a positive broader impact on this and other problems, and expect that before any mechanism based on our work were deployed, it would in fact be studied carefully.

\bibliography{main}
\bibliographystyle{icml2025}

\newpage
\appendix
\onecolumn
\section{More Prior Work}
\label{app:morepriorwork}
\paragraph{Preference elicitation and learning theory} %
\citet{blum2004preference} and \citet{lahaie2004applying} frame the preference elicitation problem in terms of learning theory, adapting tools for learning boolean functions from few queries to the problem of learning valuation functions. Many other works build in the same direction~\cite{conitzer2007eliciting,balcan2012learning,zhang2020learning,lock2022learning}.

\paragraph{Iterative auctions with machine learning}
Some auctions ask for bids in multiple rounds, such as the \emph{combinatorial clock auction (CCA)}~\cite{ausubel2006clock}, which generated over \emph{USD $20$ billion} in revenue from high-stakes spectrum auctions between $2012$ and $2014$~\cite{ausubel2017practical}.

More sophisticated iterative mechanisms attempt to make optimal queries to maximize allocative efficiency while minimizing the reporting burden on participants. 
\citet{brero2017probably,brero2018combinatorial} introduced this technique, using support vector regression to model bidders' value functions. 
\citet{brero2021workingpaper} generalized these techniques and applied them to a realistic combinatorial domain motivated by spectrum auctions, and showed better allocative efficiency than the CCA.

Other authors have improved on this work by using neural networks~\cite{weissteiner2020deep}, by learning value functions in the Fourier domain~\cite{weissteiner2022fourier}, by imposing monotonicity constraints on the neural network value functions~\cite{weissteiner2022monotone}, and by using estimations of uncertainty to choose queries in an active learning style~\cite{weissteiner2023bayesian}.
More recently, this framework has been adapted to use demand queries instead of value queries, similar to the interaction paradigm of the CCA \citep{Soumalias2024MLCCA}. 
\citet{soumalias2024pricesbidsvalueseverything} showed how both query types can be leveraged to create the most efficient ICA to date. 

\paragraph{Machine learning for mechanism design}
Machine learning, particularly deep learning, has been increasingly applied to mechanism design. \citet{dutting2019} introduced deep neural networks for auction design, improving representational flexibility. This approach, termed differentiable economics, has since been extended to budget-constrained bidders~\citep{feng2018deep}, payment minimization~\citep{tacchetti2019neural}, multi-facility location~\citep{golowich2018deep}, fairness-revenue trade-offs~\citep{kuo2020proportionnet}, two-sided matching~\citep{ravindranath2021deep}, and data markets~\citep{ravindranathdata}. Structural advancements, such as \citet{wang2024gemnet}, achieve exact strategy-proofness rather than approximate incentive compatibility. Unlike the present work, this line of work focuses on searching through the space of mechanisms, rather than using ML for preference elicitation in a fixed mechanism. Moreover, these methods primarily address settings with additive and unit-demand valuations, leaving problems with combinatorial allocations, such as the course allocation problem, largely unaddressed.

\paragraph{LLMs for Translating Natural-Language Descriptions to Formal Game Models} A related thread of research explores how LLMs can interpret free-form textual narratives and translate them into formal, game-theoretic structures \citep{daskalakis2024charting,mensfelt2024autoformalization, mensfelt2024autoformalizing, deng2025natural}. \citet{mensfelt2024autoformalization,mensfelt2024autoformalizing} map textual inputs to logic-based representations of smaller or simultaneous-move games, \citet{daskalakis2024charting} address more complex, story-like sequences by constructing extensive-form games that allow subsequent equilibrium analysis, and \citet{deng2025natural} translate game descriptions in
natural language into game-theoretic extensive-form representations. While these works focus on different problem settings than course allocation, they illustrate broader opportunities for harnessing LLMs to extract actionable, structured information from informal descriptions—an approach that aligns with our use of free-text inputs to inform combinatorial assignments.

\section{Details from \Cref{sec:PE_framework}} \label{sec:PE_framework_details}
\subsection{Ommited Proofs}
\begin{proof}[\Cref{prop:true_minimizer} Proof.]
\citet{ghosh2015} showed that, under uniform label noise with a noise rate $\eta \leq \frac{c - 1}{c}$, where $c$ is the number of classes, the minimizer of a symmetric loss function over a noiseless dataset is also the minimizer of the same loss function over a noisy dataset. Since $c = 2$ in our case, label noise is inherently uniform and this requirement becomes $\eta \leq 0.5$.

The GCE loss is known to be symmetric~\cite{Zhang18Generalized}, so the student's true valuation function minimizes the GCE loss on the noisy dataset generated by the LLM responses.
\end{proof}

\subsection{Mixed Training Algorithm and Connection to the Bradley-Terry Model} \label{sec:training_alg}
In this section, we reprint the training algorithm of \citet{soumalias2024machine} for integrating GUI reports (regression data) and CQs (ordinal data) into the training of MVNNs.

\begin{algorithm}[t!]
\caption{Mixed training for regression model $\mathcal{M}$}
\label{alg_mixed_training}
\textbf{Input}: $\{X_{reg}, y_{reg}\}$, $\{X_{class}, y_{class}\}$ \\
\textbf{Parameters}: epochs $t_{reg}$, learning rate $\eta_{reg}$, regularization parameter $\lambda_{reg}$, epochs $t_{class}$, learning rate $\eta_{class}$, regularization parameter $\lambda_{class}$ \\
\textbf{Output}: Parameters of trained ML model $\mathcal{M}$
\begin{algorithmic}[1] %
\STATE $\theta_0 \leftarrow{}$ initialize parameters of the ML model $\mathcal{M}$ \label{alg_init}
\FOR{$i = 1$ to $t_{reg}$} 
    \STATE $loss_{reg} \leftarrow 0$
    \FOR{each $(x,y)$ in $\{X_{reg}, y_{reg}\}$}
        \STATE $\hat{y} \leftarrow \mathcal{M}^{\theta_{i-1}}(x)$
        \STATE $loss_{reg} \leftarrow loss_{reg} + l_{reg}(y, \hat{y}) + \lambda_{reg} \sum{\theta_{i-1}^2}$ \label{mixed_training:reg_loss}
    \ENDFOR
    \STATE $\theta_i \leftarrow \operatorname{ADAM}(\theta_{i-1}, loss_{reg}, \eta_{reg})$ \label{alg_mixed_training_adam1}
\ENDFOR \label{alg_reg_end}
\FOR{$i = t_{reg} + 1$ to $t_{reg} + t_{class}$} \label{alg_class_start}
    \STATE $loss_{class} \leftarrow 0$
    \FOR{each $((x_1, x_2),y)$ in $\{X_{class}, y_{class}\}$}
        \STATE $\hat{y}_1 \leftarrow \mathcal{M}^{\theta_{i-1}}(x_1)$
        \STATE $\hat{y}_2 \leftarrow \mathcal{M}^{\theta_{i-1}}(x_2)$
        \STATE $\hat{y} \leftarrow \frac{1}{1 + e^{-(\hat{y}_1 - \hat{y}_2)}}$ \label{mixed_training:probability}
        \STATE $loss_{class} \leftarrow loss_{class} + l_{class}(y, \hat{y}) + \lambda_{class} \sum{\theta_{i-1}^2}$ \label{mixed_training:classification_loss}
    \ENDFOR
    \STATE $\theta_i \leftarrow \operatorname{ADAM}(\theta_{i-1}, loss_{class}, \eta_{class})$ \label{alg_mixed_training_adam2}
\ENDFOR \label{alg_class_end}
\STATE \textbf{return} $\theta_{t_{reg} + t_{class}}$ \label{alg_return}
\end{algorithmic}
\end{algorithm}

The core idea of \Cref{alg_mixed_training} is to first train the ML model on the student's GUI reports, and then finetune that training on her CQ responses. 
During the regression phase, the algorithm uses a typical regression loss, which compares the real-valued output of the model, $\mathcal{M}(\cdot)$, to the inferred value from the student's GUI input for a particular schedule (Line~\ref{mixed_training:reg_loss}).

In contrast, for CQs, we transform the model’s real-valued outputs for two schedules, $x_1$ and $x_2$, using the sigmoid function $f(x) = \frac{1}{1+e^{-x}}$. This yields a predicted probability within the $[0,1]$ range, where $\frac{1}{1 + e^{-(\hat{y}_1 - \hat{y}_2)}}$ represents the likelihood that the student finds schedule $x_1$ preferable to schedule $x_2$ (Line \ref{mixed_training:probability}).
Note that this predicted probability is exactly the one under the Bradley-Terry model \citep{bradleyterry1952rank}.
This predicted probability is then compared with the actual binary preference expressed by the student: $1$ if she preferred $x_1$, and $0$ otherwise (Line \ref{mixed_training:classification_loss}).

The training process runs for a total of $t_{reg} + t_{class}$ epochs, after which the final set of parameters is returned (Line~\ref{alg_return}).

\subsection{Monotone Value Neural Networks (MVNNs)}\label{sec:MVNNs}
In this section, we reporting the original definition of MVNNs introduced in  \citet{weissteiner2022monotone}:

\begin{definition}[MVNN]\label{def:MVNN}
		An MVNN $\mathcal{M}_i:\mathcal{X} \to \mathbb{R}_{\ge0}$  for agent $i\in N$ is defined as
            {\small
		\begin{equation}\label{eq:MVNN}
		\mathcal{M}_i (x) \coloneqq  W^{i,K_i}\varphi_{0,t^{i, K_i-1}}\left(\ldots\varphi_{0,t^{i, 1}}(W^{i,1}\left(D x\right)+b^{i,1})\ldots\right)
		\end{equation}
             }
		\begin{itemize}
		\item $K_i+2\in\mathbb{N}$ is the number of layers ($K_i$ hidden layers),
		\item $\{\varphi_{0,t^{i, k}}{}\}_{k=1}^{K_i-1}$ are the MVNN-specific activation functions with cutoff $t^{i, k}>0$, called \emph{bounded ReLU (bReLU)}:
		\begin{align}\label{itm:MVNNactivation}
		\varphi_{0,t^{i, k}}(\cdot)\coloneqq\min(t^{i, k}, \max(0,\cdot))
		\end{align}
		\item $W^i\coloneqq (W^{i,k})_{k=1}^{K_i}$ with $W^{i,k}\ge0$ and $b^i\coloneqq (b^{i,k})_{k=1}^{K_i-1}$ with $b^{i,k}\le0$ are the \emph{non-negative} weights and \emph{non-positive} biases of dimensions $d^{i,k}\times d^{i,k-1}$ and $d^{i,k}$, whose parameters are stored in $\theta=(W^i,b^i)$.
        \item\label{itm:D} $ D \coloneqq \text{diag} \left(\frac{1}{c_1},\ldots,\frac{1}{c_m}\right)$ is the linear normalization layer that ensures $D x\in [0,1]$ and is not trainable.
		\end{itemize}
\end{definition}

\begin{remark}[Initiaization]
    We always use the initialization scheme from \citet[Section~3.2 and Appednix~E]{weissteiner2023bayesian}, which offers crucial advantages over standard initialization schemes as discussed in \citet[Section~3.2 and Appednix~E]{weissteiner2023bayesian}.
\end{remark}

\subsection{Epistemic MVNNs (eMVNNs)} \label{sec:eMVNNs}
In this section, we introduce  \textit{Epistemic MVNNs (eMVNNs)}, which extend the standard MVNN architecture by incorporating an ensemble-based method to estimate epistemic uncertainty. Epistemic uncertainty reflects the model's lack of knowledge about certain regions of the input space and is crucial for guiding active learning and preference elicitation tasks. We formally define Epistemic MVNNs below.

\begin{definition}[Epistemic MVNNs]\label{def:EpistemicMVNN}
    An \emph{Epistemic MVNN} for agent $i \in N$ is defined as an ensemble of $M$ independently initialized MVNNs, denoted by 
    $\{\mathcal{M}_i^{(j)} : \mathcal{X} \to \mathbb{R}_{\ge 0}\}_{j=1}^{M}$, where each network is defined as in \Cref{def:MVNN}.
    For a given bundle $x \in \mathcal{X}$, the prediction distribution is represented by the set of outputs from the ensemble:
    \[
        \mathcal{M}_i^{E}(x) = \{\mathcal{M}_i^{(j)}(x) : j = 1, \ldots, M\}.
    \]

    The \emph{mean prediction} and \emph{epistemic uncertainty} are defined as follows:
    \begin{enumerate}
        \item \textbf{Mean prediction:}
        \[
            \hat{y}_i(x) = \frac{1}{M} \sum_{j=1}^{M} \mathcal{M}_i^{(j)}(x),
        \]

        \item \textbf{Epistemic uncertainty:}
        \[
            \sigma_i^2(x) = \frac{1}{M} \sum_{j=1}^{M} \left( \mathcal{M}_i^{(j)}(x) - \hat{y}_i(x) \right)^2.
        \]
    \end{enumerate}
    Here, $\sigma_i(x)$ denotes the standard deviation of the ensemble's predictions and serves as a measure of epistemic uncertainty.
\end{definition}

\begin{remark}
Motivated by the connection of our problem to RLHF highlighted in \Cref{sec:PE_framework}, 
we use $M = 10$ following \citep{Dwaracherla24Efficient}.
\end{remark}

\begin{remark}
    Note that the mean prediction of an eMVNN preserves the structural properties of MVNNs. Specifically:
    \begin{itemize}
        \item \textbf{Zero value for the empty bundle:} Since each MVNN in the ensemble is designed to map the empty bundle to zero, the mean prediction, as a convex combination of these outputs, also assigns a zero value to the empty bundle.
        
        \item \textbf{Monotonicity:} Each MVNN is a monotone function by construction, as its weights are constrained to be non-negative. The mean of monotone functions is itself monotone, ensuring that the mean prediction of the eMVNN is also monotone.
    \end{itemize}
\end{remark}

\newpage
\clearpage

\section{Details of LLM Proxies}

\subsection{Accuracy of LLM CQs over rounds}
\label{app:cqaccuracy}
\begin{figure}[h]
    \centering
    \includegraphics[width=0.9\linewidth]{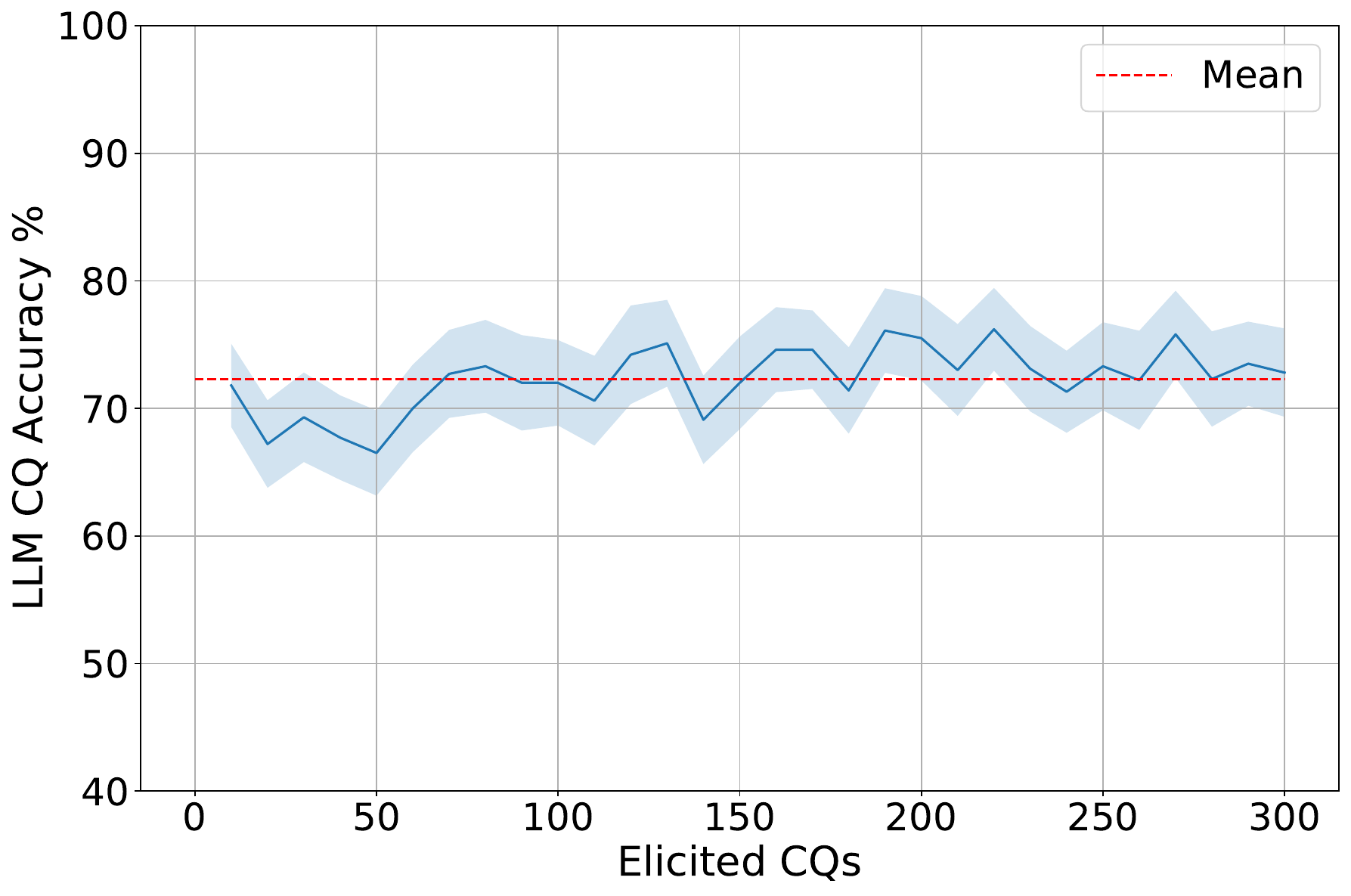}
    \vspace{-0.5cm}
    \caption{LLM CQ accuracy (evaluated using LLaMA 3.1 as the language model) as a function of the number CQs already answered by the LLM. 
    Note that later CQs are expected to be both more informative and harder to answer, as they are generated based on an acquisition function that takes into account both the current ML model, and its epistemic uncertainty.  Nonetheless, accuracy remains high even with many queries.
    Shown are averages over 100 instances including 95\% CIs. 
    }
    \label{fig:llm_cq_accuracy}
\end{figure}

{In \Cref{fig:llm_cq_accuracy}, we observe that the LLM proxy accuracy remains relatively constant throught the ML model's finetuning on CQs. 
Note that CQs are not chosen independently; instead, queries are chosen based on acquisition functions \citep{Dwaracherla24Efficient} that take into account both the student's ML model and its epistemic uncertainty. 
Thus, later queries are expected to be more informative and harder to answer since the student's value model has been already improved based on the previous queries.}

\subsection{Student Preference Conversion Prompt }
\label{appendix:prompt}

\begin{lstlisting}[basicstyle=\ttfamily\small, breaklines=true]
Please act as a student describing their course preferences for the upcoming semester. Write a detailed, first-person paragraph about your preferences based on the following information:

Course Preferences (ordered from highest to lowest value within each tier):

High Priority Courses:
{format_tier_courses(high_pref)}

Medium Priority Courses:
{format_tier_courses(medium_pref)}

Lower Priority Courses:
{format_tier_courses(low_pref)}

Course Relationships:
Overlapping Content (Substitutes) - These are courses that HARMS YOU when taken together compared to taking only one, and the more the worse:
{chr(10).join(f"- {group}" for group in substitute_groups)}

Complementary Courses (Complements) - These are courses that BENEFITS YOU HARMS YOU when taken together compared to taking only one, and the more the better:
{chr(10).join(f"- {group}" for group in complement_groups)}

Additional Constraints:
- Budget Constraint: {student.budget:.2f}
{f"- Time Gap Penalty: {student.timegap_penalty}" if student.timegap_penalty != 0 else ""}
{f"- Overload Penalty: {student.overload_penalty}" if student.overload_penalty != 0 else ""}

Please write a natural, detailed explanation of these preferences as if you were the student. Include:
1. Your strongest course interests, explaining them in order of preference within each priority tier
2. How you're thinking about course combinations, discussing specific synergies and overlaps:
   - When describing overlapping courses, explain how much the overlap affects your interest
   - When describing complementary courses, explain how much additional value you see in taking them together
3. Your overall strategy for course selection, considering both your budget constraints and the strength of course relationships
4. Any specific scheduling or workload considerations

Keep the tone conversational and authentic to how a student would describe their course preferences. Make sure to reference both your relative preferences within each tier and the specific impacts of course combinations on your overall academic plan. Your response should be three paragraph: the first paragraph should list all the top and medium tier courses you want to pick, and any courses that comes with each that might be complementary (good) or subtitutes (bad). The second paragraph should clearly and detailedly state ALL OF THE bundles of courses that are complements and bundles of courses that are substitutes for you; for each, explain how much it hurts you when you take different numbers of courses from that bundle. You should cover all the complement bundles and substitute bundles given to you - leave nothing out. The thrid paragraph concludes and highlights other concerns. Aim for qualitatively detailed description and avoid saying exact numerical values in the output (e.g. it's fine to say 'taking x and x together is suboptimal, taking x, x and x together even more suboptimal, and taking x,x,x, and x together should really be avoided, (so on so forth until the size of the complement/substitute set).', but not fine to say 'My value for course 16 is 102.45', or 'x and x together decreases my utility by 43%
"""
\end{lstlisting}

\subsection{Comparison Query Prompt }
\label{appendix:cqprompt}

\begin{lstlisting}[basicstyle=\ttfamily\small, breaklines=true]

Based on these student preferences:

{student_preferences_text}

Compare:
Bundle A: {format_bundle(bundle1)}
Bundle B: {format_bundle(bundle2)}
to choose the better bundle.

Please ignore budget constraint if it's mentioned in the preference - pretend it doesn't exist.

Your response must use these EXACT tags below, and ONLY include the tags, end your response after that. The text between tags should be concise.

'''
<PREFERENCES>
Bundle A: [First recall the courses in Bundle, then list matching preferences, e.g. Bundle A contains Courses X, X, (list all courses in Bundle A). Course X is high preference, Course X is mid preference, Course X is low preference]
Bundle B: [First recall the courses in Bundle, then list matching preferences, e.g. Bundle B contains Courses X, X, (list all courses in Bundle B). Course X is high preference, Course X is mid preference, Course X is low preference]
</PREFERENCES>

<COMPLEMENTS>
Bundle A: [First recall the courses in Bundle, then list complementary relationships with magnitudes or "None", e.g. Bundle A contains Courses X, X, (list all courses in Bundle A). Course X and Course X are complements which helps moderately when taken together, Course X, Course X, and Course X are complements and helps significantly when taken together]
Bundle B: [First recall the courses in Bundle, then list complementary relationships with magnitudes or "None", e.g. Bundle B contains Courses X, X, (list all courses in Bundle B). Course X and Course X are complements which helps moderately when taken together, Course X, Course X, and Course X are complements and helps significantly when taken together]
</COMPLEMENTS>

<SUBSTITUTES>
Bundle A: [First recall the courses in Bundle, then list substitute relationships with magnitudes or "None", e.g. Bundle A contains Courses X, X, (list all courses in Bundle A). Course X and Course X are substitutes which harms moderately when taken together, Course X, Course X, and Course X are substitutes and harms significantly when taken together]
Bundle B: [First recall the courses in Bundle, then list substitute relationships with magnitudes or "None", e.g. Bundle B contains Courses X, X, (list all courses in Bundle B). Course X and Course X are substitutes which harms moderately when taken together, Course X, Course X, and Course X are substitutes and harms significantly when taken together]
</SUBSTITUTES>

<REASONING>
[Provide your concise reasoning in a few sentences, e.g. From the above, in terms of preferences, Bundle X is better. In terms of the presence and magnitude of complements, Bundle X is better. In terms of magnitude and precense of substitutes, bundle X is better. Considering the tradeoffs, Bundle X is better.]
</REASONING>

<CHOICE>Bundle X</CHOICE>
'''
\end{lstlisting}

\subsection{Example of A Synthetic Student}
\label{appendix:examplestudent}

\begin{lstlisting} [basicstyle=\ttfamily\small, breaklines=true]
Course Preferences (ordered from highest to lowest value within each tier):

High Priority Courses:
   1. Course 20 (value: 117.96)
   2. Course 19 (value: 115.30)

Medium Priority Courses:


Lower Priority Courses:
   1. Course 21 (value: 56.26)
   2. Course 7 (value: 55.99)
   3. Course 14 (value: 55.73)
   4. Course 17 (value: 55.71)
   5. Course 25 (value: 54.02)
   6. Course 2 (value: 52.66)
   7. Course 22 (value: 52.37)
   8. Course 3 (value: 52.07)
   9. Course 6 (value: 51.39)
   10. Course 8 (value: 50.80)
   11. Course 16 (value: 50.80)
   12. Course 1 (value: 49.36)
   13. Course 9 (value: 46.22)
   14. Course 13 (value: 45.73)
   15. Course 23 (value: 45.31)
   16. Course 15 (value: 45.18)
   17. Course 18 (value: 43.96)
   18. Course 4 (value: 43.71)
   19. Course 24 (value: 42.64)
   20. Course 5 (value: 42.55)
   21. Course 10 (value: 42.34)
   22. Course 12 (value: 41.71)
   23. Course 11 (value: 40.76)

Course Relationships:
Overlapping Content (Substitutes) - These are courses that HARMS YOU when taken together compared to taking only one, and the more the worse:
- Courses 19, 13, 16 overlap in content. Taking any two reduces their combined value by 40%
- Courses 20, 25, 3, 17, 11 overlap in content. Taking any two reduces their combined value by 40%

Complementary Courses (Complements) - These are courses that BENEFITS YOU HARMS YOU when taken together compared to taking only one, and the more the better:
- Courses 19, 12 complement each other. Taking any two increases their combined value by 40%
- Courses 20, 12, 5, 9 complement each other. Taking any two increases their combined value by 40%

Additional Constraints:
- Budget Constraint: 1.00
\end{lstlisting}

\subsection{Synthetic Student Response and Effect of Conciseness on Performance}
\label{sec:appendix-conciseness}

\subsubsection{Natural Language Descriptions of Preferences}

The natural language descriptions provided by the student are a critical component of our framework. Below is an example of a detailed, baseline response using the prompt in Section \ref{appendix:prompt}:

\textit{``I've got a pretty clear idea of what I'm looking for in the upcoming semester. For high-priority courses, I'm really set on taking Course 25 and Course 15 - both of these courses are at the top of my list, and I think they'd be a great foundation for the semester. In terms of medium-priority courses, I'd love to take Course 10, Course 24, Course 14, Course 21, Course 5, Course 4, Course 7, Course 3, Course 22, Course 19, Course 1, Course 23, Course 13, Course 16, Course 8, Course 12, Course 18, Course 2, Course 17, Course 9, Course 6, Course 20, and Course 11. I'm not as attached to these courses, but I think they'd all be great additions to my schedule. Some of these courses have potential complements or substitutes, though - for example, Course 25, Course 5, Course 1, Course 19, and Course 24 all overlap in content, which means taking any two of these courses reduces their combined value by a significant amount, and taking all five would be really suboptimal. On the other hand, Course 25, Course 3, Course 11, and Course 22 complement each other, so taking any two of these courses increases their combined value, and taking all four would be a real bonus.}

\textit{When it comes to course combinations, I'm thinking about both the complements and substitutes. The overlapping courses I mentioned earlier are definitely something to watch out for - taking any two of Course 25, Course 5, Course 1, Course 19, and Course 24 together is suboptimal, taking any three is even worse, and taking all five is basically a no-go. On the other hand, taking any two of Course 25, Course 3, Course 11, and Course 22 together is a good thing, and taking all four would be even better. I'm also aware that Course 15, Course 12, and Course 4 overlap in content, which means taking any two of these courses together is suboptimal, and taking all three is really suboptimal. However, taking any two of Course 15, Course 17, and Course 23 together is a good thing, and taking all three would be even better. I'm trying to balance out my course load to make the most of these relationships - I don't want to take too many courses that overlap in content, but I also want to take advantage of the complements.}

\textit{Overall, my strategy for course selection is to prioritize my high-priority courses and then try to balance out my medium-priority courses to take advantage of the complements and avoid the substitutes. I'm also thinking about my budget constraint - I don't want to take on too much of a financial burden, so I'm trying to be mindful of the costs associated with each course. In terms of scheduling and workload, I'm a bit concerned about how I'll fit everything in - I've got a lot of courses on my list, and I don't want to overload myself. I'm hoping to work with my advisor to figure out a schedule that makes sense and allows me to make the most of my course load."
}

This comprehensive response effectively conveys the student's preferences and provides sufficient detail to inform our mechanism. However, in practice, students might not always provide responses at this level of detail. To understand how response conciseness impacts our framework's performance, we introduce two alternative levels of brevity:

\begin{itemize}
    \item \textbf{Moderately Brief:} Responses are prompted to be limited to approximately 180 words. While the LLM may not always adhere exactly to this constraint, the intent is to capture a moderately detailed response.
    \item \textbf{Highly Brief:} Responses are prompted to be limited to approximately 80 words. Again, the LLM might not strictly follow this word limit, but the responses are qualitatively much shorter and less detailed.
\end{itemize}

These levels are tested by asking the LLM proxy simulating student responses to adhere to the respective word constraints. Below are examples of the moderately brief and highly brief responses:

\textbf{Moderately Brief ($\sim$180 words):}
\textit{``I'm really interested in taking Course 21, as it has the highest value among all the options. If I can't take it, my next choice would be Course 13. I also want to take Course 8, which is my only medium-priority course. Within the lower-priority tier, I'm interested in Course 6, Course 16, and Course 2. I'm aware that Course 13 and Course 3 have overlapping content, so taking both would be suboptimal. Similarly, taking Course 21, Course 11, Course 2, and Course 18 together would also be suboptimal. On the other hand, taking Course 13, Course 8, Course 24, and Course 5 together would be beneficial, as it would increase the combined value. The same applies to taking Course 21, Course 12, Course 8, Course 5, and Course 22 together.}

\textit{In terms of bundles, I'd like to avoid taking Course 13 and Course 3 together, as it would be suboptimal. Taking both would be bad, taking all three (Course 13, Course 3, and Course 2) would be even worse, and taking all four (Course 13, Course 3, Course 2, and Course 10) would be really bad. I'd also like to avoid taking Course 21, Course 11, Course 2, and Course 18 together, as it would be suboptimal. Taking any two would be bad, taking all three would be worse, and taking all four would be really bad. On the other hand, taking Course 13, Course 8, Course 24, and Course 5 together would be beneficial, as it would increase the combined value."}

\textbf{Highly Brief ($\sim$80 words):}
\textit{``My top priority courses are Course 11 and Course 18. I'm interested in Course 19, Course 16, and Course 22, but they're not as crucial. I also consider Course 21, Course 13, and Course 20 as potential options. For Course 6, Course 25, and Course 5, I'm not as enthusiastic. Course 23 and Course 2 are decent choices, but I'm not sure about Course 17, Course 15, and Course 1. Course 8 and Course 3 are somewhat appealing, but Course 9 and Course 24 are less desirable. Course 14, Course 12, Course 7, and Course 4 are not my first choices."}

These variations allow us to test the robustness of our framework when students provide less detailed descriptions.

\subsubsection{Impact of Brevity on Value and Correctness}

Table~\ref{tab:conciseness_performance} summarizes the performance for each level of brevity, including normalized allocated bundle values, correctness rates, and statistical significance metrics. The correctness rates show only minor differences between the baseline (72.3\%), moderately brief (71.7\%), and highly brief (68.04\%) versions. However, these differences in correctness significantly impact the normalized allocated bundle value, dropping from 119\% for the baseline to 110\% for the moderately brief version and 106\% for the highly brief version. Notably, while the moderately brief version still demonstrates statistically significant improvement, the highly brief version does not.

\begin{table}[ht]
    \centering
    \begin{tabular}{l c c c c c}
        \toprule
        \textbf{Brevity Level} & \textbf{Normalized Value} & \textbf{CQ Correctness} & \multicolumn{2}{c}{\textbf{\% of Runs}} & \textbf{P-value} \\
        \cline{4-5}
         & \(\pm\ CI\) & \(\%\) & \textbf{Better} & \textbf{Worse} \\
        \midrule
        Baseline & 119.34 \(\pm 6.50\) & 72.3 & 74 & 26 & $6.84 \times 10^{-8}$ \\
        Moderately Brief & 110.35 \(\pm 8.52\) & 71.7 & 58 & 42 & $0.021$ \\
        Highly Brief & 106.54 \(\pm 8.29\) & 68.04 & 52 & 48 & $0.129$ \\
        \bottomrule
    \end{tabular}
    \caption{Summary of performance metrics for varying levels of brevity.}
    \label{tab:conciseness_performance}
\end{table}

We evaluate the framework's performance by measuring both the normalized allocated bundle value and the correctness of the LLM in answering comparison queries. Figure~\ref{fig:normalized_value} plots the changes in normalized allocated bundle value as the number of elicited CQs increases, while Figure~\ref{fig:correctness_rate} shows the correctness of the LLM at answering comparison queries.

\begin{figure}[ht]
    \centering
    \includegraphics[width=0.8\linewidth]{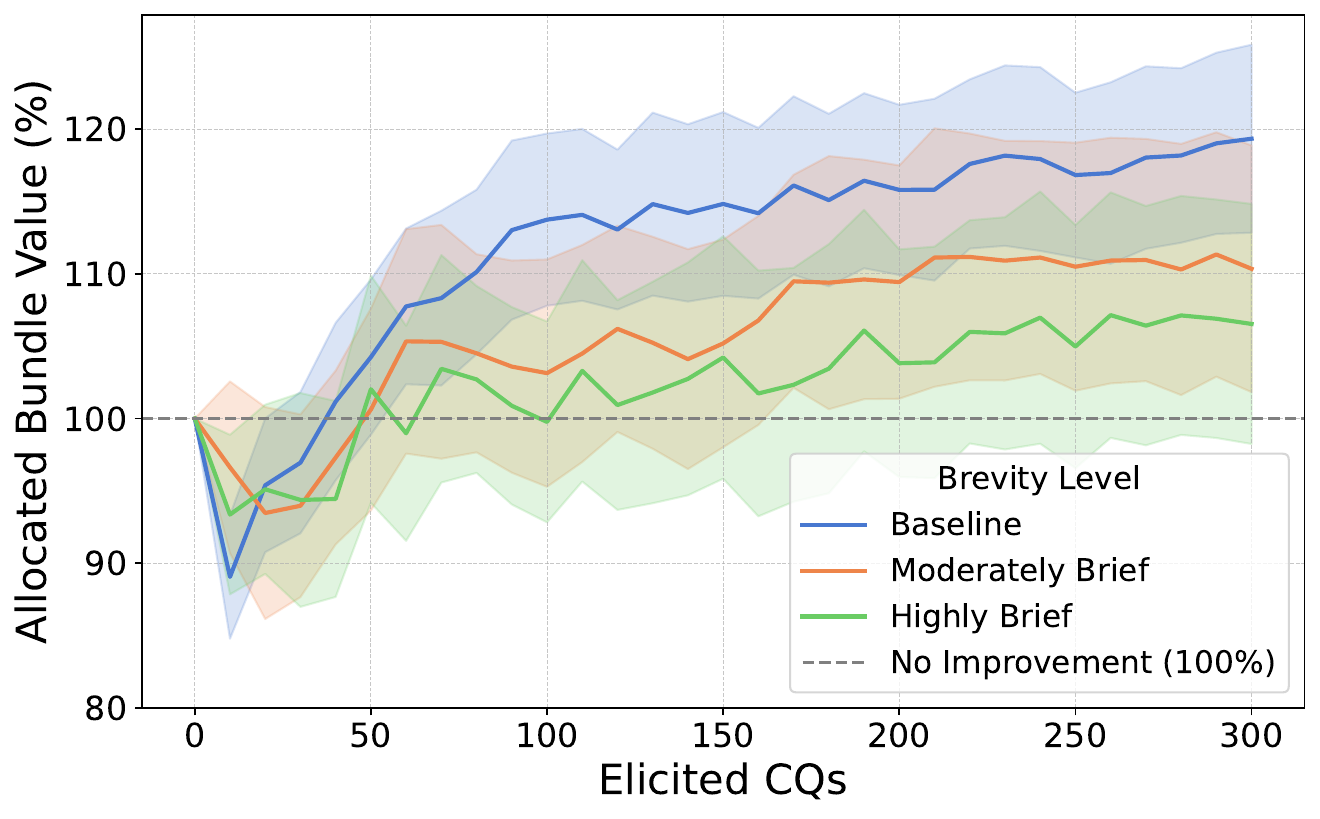}
    \caption{Normalized allocated bundle value as a function of elicited CQs for different brevity levels.}
    \label{fig:normalized_value}
\end{figure}

\begin{figure}[ht]
    \centering
    \includegraphics[width=0.8\linewidth]{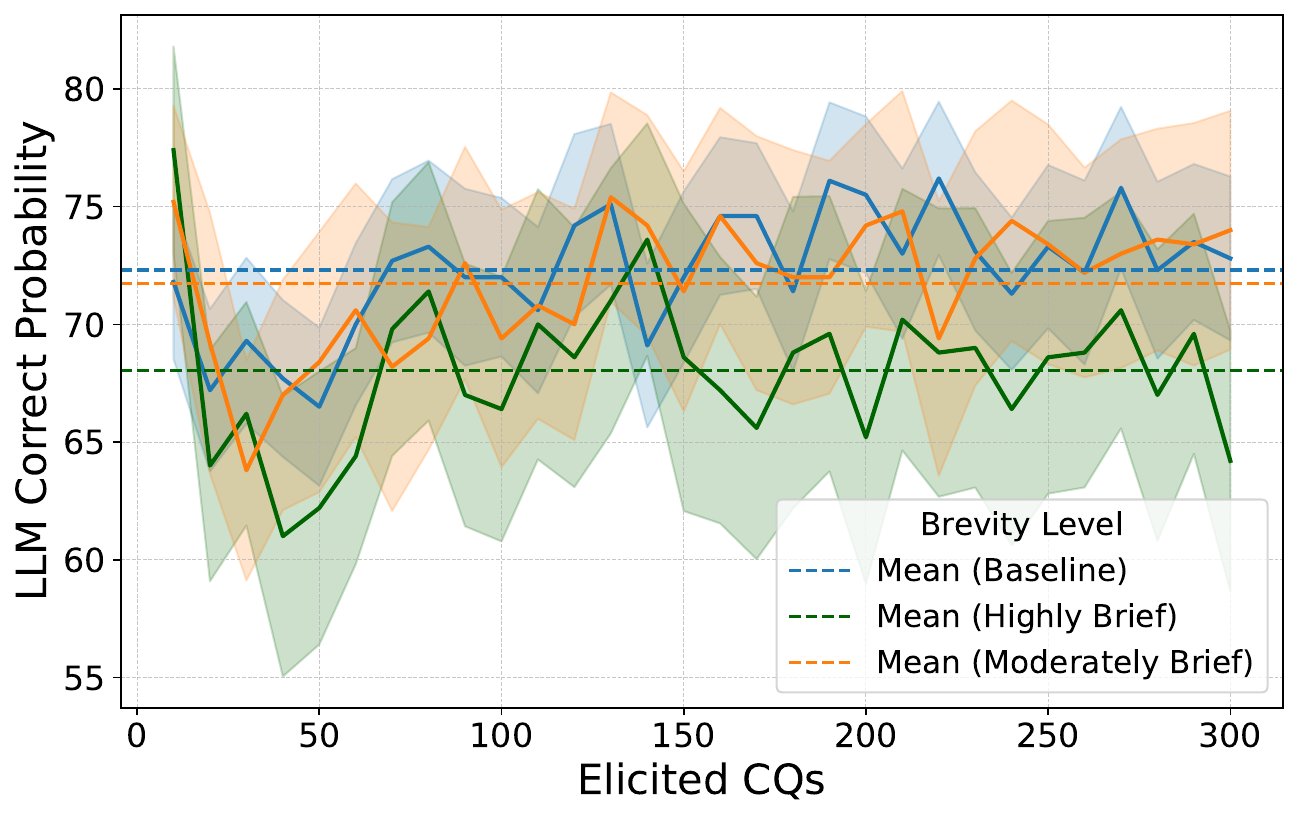}
    \caption{Correctness rate of the LLM at answering comparison queries for different brevity levels.}
    \label{fig:correctness_rate}
\end{figure}

\subsubsection{Discussion}

These results suggest that our framework is robust to a degree when students provide less detailed responses. However, as the level of brevity increases, the framework's performance deteriorates. This is intuitive since the mechanism relies on the student's textual descriptions to answer comparison queries accurately. To mitigate this, incentivizing students to provide thorough responses or leveraging more advanced language models with better correctness guarantees can help maintain high performance.

\subsection{Additional Figures}

\label{sec:additional_results}

In this section, we present supplementary analyses and results that provide additional insights into our experimental findings. Figure~\ref{fig:llm_cq_correctness_cot} demonstrates the impact of Chain-of-Thought (CoT) reasoning on the LLM's accuracy in processing constraint queries. This analysis complements our main results by showing the direct relationship between CoT implementation and improved model performance.

\begin{figure}[h]
    \centering
    \includegraphics[width=\linewidth]{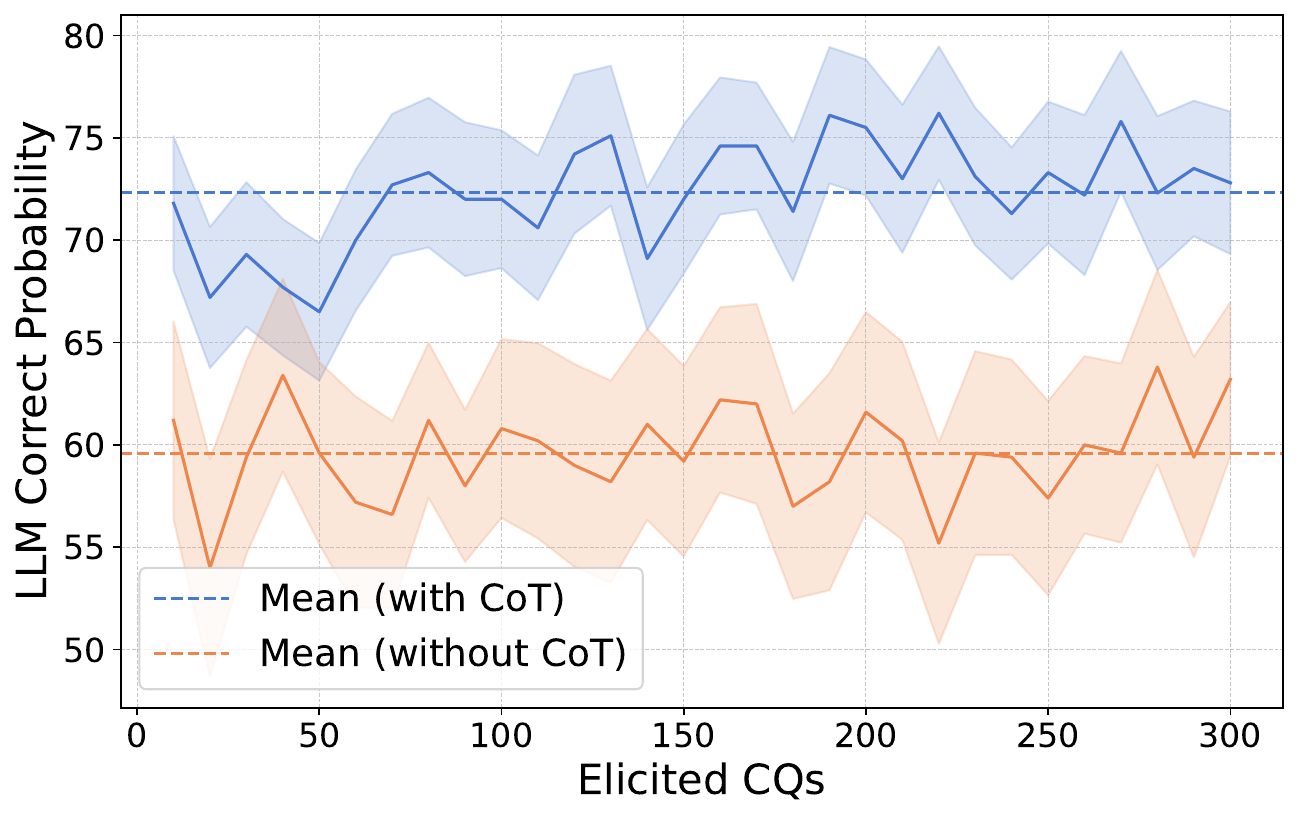}
    \vspace{-0.5cm}
    \caption{Comparison of the LLM Correct Probability with and without the Chain-of-Thought (CoT) reasoning approach using LLaMA 3.1 8b model. The use of CoT increases the correctness probability from 59\% to over 70\%, enabling the pipeline to achieve statistically significant value improvements.
    }
    \label{fig:llm_cq_correctness_cot}
\end{figure}

\subsection{Discussion}

These results suggest that our framework is robust to a degree when students provide less detailed responses. However, as the level of conciseness increases, the framework's performance deteriorates. This is intuitive since the mechanism relies on the student's textual descriptions to answer comparison queries accurately. To mitigate this, incentivizing students to provide thorough responses or leveraging more advanced language models with better correctness guarantees can help maintain high performance.

\subsection{Examples of comparison query output} 
\label{appendix:output}

The following is a correct answer for one CQ.
\begin{lstlisting} [basicstyle=\ttfamily\small, breaklines=true]
<PREFERENCES> Bundle A: Bundle A contains Courses 1, 9, 12, 22, 24. Course 22 is high preference, Course 9 is low preference, Courses 1, 12, 24 are not mentioned in preferences. Bundle B: Bundle B contains Courses 8, 9, 12, 16, 22. Course 22 is high preference, Course 9 is low preference, Courses 8, 12, 16 are not mentioned in preferences. </PREFERENCES>

<COMPLEMENTS> Bundle A: Bundle A contains Courses 1, 9, 12, 22, 24. Course 22 and Course 12 are complements which helps significantly when taken together, Course 9 and Course 22 are complements which helps moderately when taken together. Bundle B: Bundle B contains Courses 8, 9, 12, 16, 22. Course 22 and Course 12 are complements which helps significantly when taken together, Course 9 and Course 16 are complements which helps significantly when taken together. </COMPLEMENTS>

<SUBSTITUTES> Bundle A: Bundle A contains Courses 1, 9, 12, 22, 24. Course 22 and Course 12 are substitutes which harms moderately when taken together, Course 9 and Course 12 are substitutes which harms moderately when taken together. Bundle B: Bundle B contains Courses 8, 9, 12, 16, 22. Course 22 and Course 12 are substitutes which harms moderately when taken together, Course 9 and Course 12 are substitutes which harms moderately when taken together. </SUBSTITUTES>

<REASONING> From the above, in terms of preferences, both bundles are similar since both contain high and low preference courses. In terms of the presence and magnitude of complements, Bundle B is better as it has two significant complements. In terms of magnitude and presence of substitutes, both bundles are similar. Considering the tradeoffs, Bundle B is better. </REASONING> 

<CHOICE>Bundle B</CHOICE>
\end{lstlisting}

The following is an incorrect answer for the CQ. The difference between the two bundles is minimal, and the LLM simply compared the number of areas where each bundle excels.
\begin{lstlisting} [basicstyle=\ttfamily\small, breaklines=true]
 <PREFERENCES> Bundle A: Bundle A contains Courses 9, 12, 14, 17, 21. Course 14 is high preference, Course 9 is low preference. Bundle B: Bundle B contains Courses 2, 14, 17, 21, 22. Course 14 and Course 22 are high preference. </PREFERENCES> 
 
 <COMPLEMENTS> Bundle A: Bundle A contains Courses 9, 12, 14, 17, 21. Course 14 and Course 9 are complements which helps significantly when taken together. Course 12 and Course 14 are complements which helps significantly when taken together. Bundle B: Bundle B contains Courses 2, 14, 17, 21, 22. Course 14 and Course 22 are complements which helps significantly when taken together. </COMPLEMENTS> 
 
 <SUBSTITUTES> Bundle A: Bundle A contains Courses 9, 12, 14, 17, 21. Course 14 and Course 12 are substitutes which harms moderately when taken together. Bundle B: Bundle B contains Courses 2, 14, 17, 21, 22. Course 14 and Course 22 are substitutes which harms significantly when taken together. </SUBSTITUTES> 
 
 <REASONING> From the above, in terms of preferences, Bundle B is better. In terms of the presence and magnitude of complements, Bundle A is better. In terms of magnitude and presence of substitutes, Bundle A is better. Considering the tradeoffs, Bundle A is better. </REASONING> 
 
 <CHOICE>Bundle A</CHOICE> 
\end{lstlisting}

\section{Details from \Cref{sec:experimets}}

\subsection{Neural Network Architecture and Training Algorithm Details} \label{sec:mvnns_and_learning_alg}

Similar to \citet{soumalias2024machine}, for the baseline architecture of the student's ML model, we used MVNNs \citep{weissteiner2022monotone}. 
These are a special neural network architecture designed to model monotone value functions in combinatorial assignment problem. 
The training algorithm that we used to integerate both the student's GUI reports and the LLM-answered CQs is the algorithm proposed in \citet{soumalias2024machine}, as that algorithm was designed to trained on exactly the same dataset types as we have.

\subsection{Hyperparameter Optimization (HPO)} \label{sec:HPO}
In this section, we describe our hyperparameter optimization (HPO) methodology and the parameter ranges used.

Given the computational cost of querying LLMs, we employed \textit{simulated LLMs} during the HPO process. First, we measured the accuracy of our default LLM architecture (as shown in \Cref{fig:llm_cq_accuracy}). During HPO, we simulated LLM responses by providing both correct and incorrect answers with probabilities that reflected the measured accuracy, assuming i.i.d. mistakes across queries.

The training algorithm described in \citet{soumalias2024machine} begins by training a student's ML model on her reports in the mechanism's original language, creating a cardinal dataset from those reports. The model is then fine-tuned based on student-answered CQs. In our framework, however, LLMs perform this task instead of students. Therefore, we optimized hyperparameters related to training on CQs and the acquisition function responsible for selecting which CQs the LLM answers.

Our optimization goal was to maximize the student's allocated bundle value, which serves as the primary performance metric. For each configuration tested, we ran separate instances and averaged performance over 10 runs, using different seeds from those employed in the experiments presented in \Cref{sec:experimets}.

The best-performing configuration and the full range of hyperparameters are provided in \Cref{tab:HPO_ranges}.

\begin{table}[h]
\centering
\resizebox{0.85\textwidth}{!}{
\begin{tabular}{lll}
\toprule
\multicolumn{1}{l}{\textbf{Hyperparameter}} &   \multicolumn{1}{l}{\textbf{HPO-Range}} & \multicolumn{1}{l}{\textbf{Winning Configuration}} \\
\midrule
LLM CQ Batch Size            & {[}1, 2, 4, 8, 16, 32{]}                          & 1                           \\
LLM CQ Epochs                & {[}2, 5, 10, 20, 50, 100, 200, 500, 1000{]}       & 10                           \\
LLM CQ Learning Rate         & (0.0001, 0.1) & 0.01                                            \\
LLM CQ Weight Decay          & {[}0, 1e-8, 1e-6, 1e-5, 1e-4, 1e-3, 1e-2, 1e-1{]} & 0.01                         \\
LLM CQ Gradient Clipping     & {[}0.01, 0.02, 0.05, 0.1, 0.2, 0.5, 1, 2, 5, 10{]} & 0.2                          \\
LLM Generalized Cross Entropy $q$ & (0.001, 1.0) & 0.3                              \\
\midrule
LLM CQs                      & [200, 300, 500]                                     & 500                          \\
LLM Acquisition Function     & [Boltzmann, Infomax, Random, Double TS] & Double TS \\
\bottomrule
\end{tabular}
}
\vskip 0.1cm
    \caption{HPO ranges and winning parameters for our framework.}
    \label{tab:HPO_ranges}
\end{table}

\subsection{Detailed Learning Experiments} \label{sec:learning_comparison_details}
In this section, we report a more detailed version of \Cref{tab:learning_comparison} from section \Cref{sec:experiment_results_learning}. 

We observe that our framework causes a dramatic improvement in learning performance for \textit{all} metrics and quantiles tested.

\begin{table}[ht]
\centering
\begin{tabular}{l r c c c c c}
\toprule
\textbf{Metric} & \multicolumn{2}{c}{\textbf{Whole Test Set}} & \multicolumn{2}{c}{\textbf{Top 10\% Quantile}} & \multicolumn{2}{c}{\textbf{Top 5\% Quantile}} \\
\cmidrule(lr){2-3} \cmidrule(lr){4-5} \cmidrule(lr){6-7}
& \textbf{Before} & \textbf{After} & \textbf{Before} & \textbf{After} & \textbf{Before} & \textbf{After} \\
\midrule
$\text{MAE}_C$      & $46.14$  & \ccell $30.69$ & $51.23$ & \ccell $27.96$ & $51.15$  & \ccell  $26.19$  \\
$\text{KT}$         & $0.266$  & \ccell $0.389$ & $0.121$ & \ccell $0.283$ &  $0.102$ & \ccell $0.258$ \\
$\text{MSE}_C$      & $3475$   & \ccell $1942$  & $4172$  & \ccell $1500$  & $4139.6$ & \ccell $1309.4$ \\
$R^2_C$             & $-0.62$  & \ccell $0.22$  & $-2.50$ & \ccell $0.14$  & $-3.52$  & \ccell $0.11$ \\
\bottomrule
\end{tabular}
\caption{Comparison of learning metrics before and after the LLM-elicited CQs. Each metric is reported for the whole test set, the top 10\% quantile, and the top 5\% quantile in terms of student value. 
Metrics include MAE, Kendall’s Tau (KT), Mean Squared Error Centered (MSE-C), and R-squared centered ($R^2_C$). 
Shown are averages over $100$ runs.
Winners based on a paired t-test with $\alpha=1\%$ are marked in grey. \mjc{tough to squash this table to fit in 1 col. but too small for table* two-col table. we should consider breaking it up somehow if possible.}}
\label{tab:learning_comparison_full}
\end{table}

\end{document}